\documentclass[english,twoside,11pt]{article}
\usepackage{jmlr2e}
\usepackage{amsmath}

\usepackage{sidecap}
\usepackage{subfigure}

\iftrue
\fi

\newcommand{\E}{\mathbb{E}}
\newcommand{\sigm}{\mathop{\mathrm{sigmoid}}}

\usepackage{babel}

\begin{document}

\firstpageno{1}

\title{What Regularized Auto-Encoders Learn from the Data Generating Distribution}

\author{
\noindent Guillaume Alain and Yoshua Bengio\\
guillaume.alain@umontreal.ca, yoshua.bengio@umontreal.ca\\
\\
Department of Computer Science and Operations Research\\
University of Montreal\\
Montreal, H3C 3J7, Quebec, Canada
}

\editor{  } 

\maketitle

\begin{abstract}%
  What do auto-encoders learn about the underlying data generating distribution?  Recent work suggests that some auto-encoder variants do a good job of capturing the local manifold structure of data.  This paper clarifies some of these previous observations by showing that minimizing a particular form of regularized reconstruction error yields a reconstruction function that locally characterizes the shape of the data generating density. We show that the auto-encoder captures the score (derivative of the log-density with respect to the input). It contradicts previous interpretations of reconstruction error as an energy function. Unlike previous results, the theorems provided here are completely generic and do not depend on the parametrization of the auto-encoder: they show what the auto-encoder would tend to if given enough capacity and examples. These results are for a contractive training criterion we show to be similar to the denoising auto-encoder training criterion with small corruption noise, but with contraction applied on the whole reconstruction function rather than just encoder. Similarly to score matching, one can consider the proposed training criterion as a convenient alternative to maximum likelihood because it does not involve a partition function.  Finally, we show how an approximate Metropolis-Hastings MCMC can be setup to recover samples from the estimated distribution, and this is confirmed in sampling experiments.
\end{abstract}

\section{Introduction}

Machine learning is about capturing aspects of the unknown distribution
from which the observed data are sampled (the {\em data-generating
distribution}). For many learning algorithms and in particular
in {\em manifold learning}, the focus is on identifying
the regions (sets of points) in the space of examples where
this distribution concentrates, i.e., which configurations of the
observed variables are plausible. 

Unsupervised {\em representation-learning} algorithms attempt to
characterize the data-generating distribution through the discovery of a
set of features or latent variables whose variations capture most of the
structure of the data-generating distribution.  In recent years, a number
of unsupervised feature learning algorithms have been proposed that are
based on minimizing some form of {\em reconstruction error}, such as
auto-encoder and sparse coding
variants~\citep{Olshausen-97,Bengio-nips-2006,ranzato-07,
  Jain-Seung-08,ranzato-08,VincentPLarochelleH2008,Koray-08,
  Rifai+al-2011,Dauphin-et-al-NIPS2011,gregor-nips-11}.  An auto-encoder
reconstructs the input through two stages, an encoder function $f$ (which
outputs a learned representation $h=f(x)$ of an example $x$) and a decoder
function $g$, such that $g(f(x))\approx x$ for most $x$ sampled from the
data-generating distribution.  These feature learning algorithms can be
{\em stacked} to form deeper and more abstract representations. {\em Deep
  learning} algorithms learn multiple levels of representation, where the
number of levels is data-dependent.  There are theoretical arguments and
much empirical evidence to suggest that when they are well-trained, deep
learning
algorithms~\citep{Hinton06,Bengio-2009,HonglakL2009,Salakhutdinov2009,Bengio+Delalleau-ALT-2011-short,Bengio-et-al-icml2013}
can perform better than their shallow counterparts, both in terms of
learning features for the purpose of classification tasks and for
generating higher-quality samples.

Here we restrict ourselves to the case of continuous inputs
$x \in \mathbb{R}^d$ with the data-generating distribution being associated with an
unknown {\em target density} function, denoted $p$.
Manifold learning
algorithms assume that $p$ is concentrated in regions of lower 
dimension~\citep{Cayton-2005,Narayanan+Mitter-NIPS2010-short},
i.e., the training examples are by definition located very close to
these high-density manifolds.  In that context, the core objective of
manifold learning algorithms is to identify where the density concentrates.

Some important questions remain concerning many of feature learning
algorithms based on reconstruction error. 
Most importantly, {\em what is their training criterion learning about the input
  density}? Do these algorithms implicitly learn about the whole density or
only some aspect? 
If they capture the essence of the target
density, then can we formalize that link and in particular exploit it
to {\em sample from the model}? 
The answers may help to establish that these algorithms actually learn
{\em implicit density} models, which only define a density indirectly, e.g., 
through the estimation of statistics or through a generative procedure.
These are the questions to which this paper contributes.

The paper is divided in two main sections, along with detailed appendices with
proofs of the theorems.
Section 2 makes a direct link between denoising
auto-encoders~\citep{VincentPLarochelleH2008} and contractive 
auto-encoders~\citep{Rifai+al-2011}, justifying the interest
in the contractive training criterion studied in the rest of the paper.
Section 3 is the main contribution and regards the following question: when minimizing
that criterion, {\em what does an auto-encoder learn about 
the data generating density}? The main answer is that it estimates the {\em score}
(first derivative of the log-density), i.e., the direction in which
density is increasing the most, which also corresponds to the {\em local mean},
which is the expected value in a small ball around the current location.
It also estimates the Hessian (second derivative of the log-density).

Finally, Section 4 shows how having access to an estimator
of the score can be exploited to estimate energy differences,
and thus perform approximate MCMC sampling. This is achieved
using a Metropolis-Hastings MCMC in which the energy differences
between the proposal and the current state are approximated using
the denoising auto-encoder. Experiments on artificial datasets
show that a denoising auto-encoder can recover a good estimator
of the data-generating distribution, when we compare the samples
generated by the model with the training samples, projected into
various 2-D views for visualization.

\section{Contractive and Denoising Auto-Encoders}

Regularized auto-encoders (see~\citet{Bengio-Courville-Vincent-arxiv-2012} for a review and a longer
exposition) capture the structure of the training distribution thanks
to the productive opposition between reconstruction error and a regularizer.
An auto-encoder maps inputs $x$ to an internal representation (or code) $f(x)$
through the encoder function $f$, and then maps back $f(x)$ to the input space through
a decoding function $g$. The composition of $f$ and $g$ is called the reconstruction
function $r$, with $r(x)=g(f(x))$, and a reconstruction loss function $\ell$
penalizes the error made, with $r(x)$ viewed as a prediction of $x$.
When the auto-encoder is regularized, e.g., via a sparsity regularizer,
a contractive regularizer (detailed below), or a denoising form of regularization
(that we find below to be very similar to a contractive regularizer), the
regularizer basically attempts to make $r$ (or $f$) as simple as possible,
i.e., as constant as possible, as unresponsive to $x$ as possible. It means
that $f$ has to throw away some information present in $x$, or at least
represent it with less precision. On the other hand, to make reconstruction
error small on the training set, examples that are neighbors on a high-density manifold must
be represented with sufficiently different values of $f(x)$ or $r(x)$. Otherwise,
it would not be possible to distinguish and hence correctly reconstruct these examples. 
It means that the derivatives of $f(x)$ or $r(x)$ in the $x$-directions along the manifold
must remain large, while the derivatives (of $f$ or $r$) in the $x$-directions orthogonal to
the manifold can be made very small. This is illustrated in Figure~\ref{fig:manifold}.
In the case of principal components
analysis, one constrains the derivative to be exactly 0 in the directions
orthogonal to the chosen projection directions, and around 1 in the
chosen projection directions. In regularized auto-encoders, $f$ is non-linear,
meaning that it is allowed to choose different principal directions (those
that are well represented, i.e., ideally the manifold tangent directions)
at different $x$'s, and this allows a regularized auto-encoder with non-linear
encoder to capture non-linear manifolds.
Figure \ref{fig:1D-autoencoder} illustrates the extreme case when
the regularization is very strong ($r(\cdot)$ wants to be nearly
constant where density is high) in the special case where the
distribution is highly concentrated at three points (three training
examples). It shows the compromise between obtaining the identity
function at the training examples and having a flat $r$ near the
training examples, yielding a vector field $r(x)-x$ that points towards the
high density points.

\begin{figure}
\vspace*{-5mm}
\centering
\includegraphics[scale=1.2]{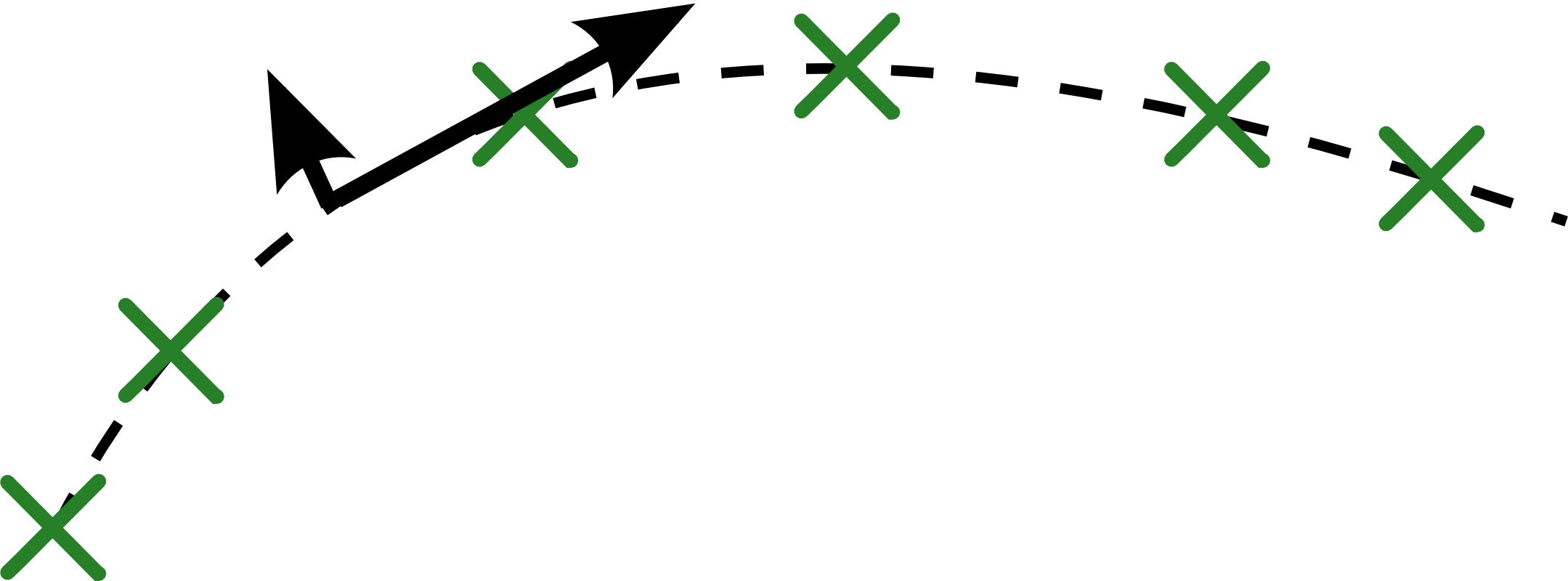}
\vspace{-3mm}
\caption{Regularization forces the auto-encoder to become less sensitive
to the input, but minimizing reconstruction error forces it to remain sensitive
to variations along the manifold of high density. Hence the representation
and reconstruction end up capturing well variations on the manifold while
mostly ignoring variations orthogonal to it.}
\label{fig:manifold}
\vspace{-2mm}
\end{figure}

\begin{figure}
\centering
\vspace*{-2mm}
\includegraphics[scale=0.8]{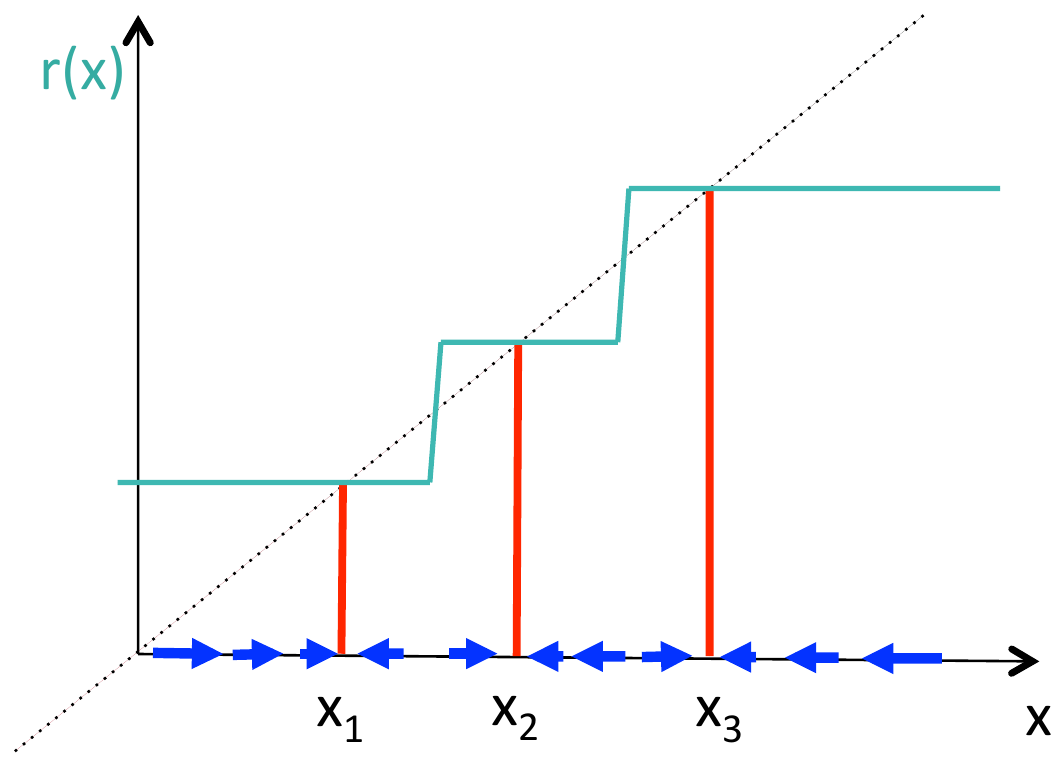}
\vspace*{-2mm}
\caption{\small The reconstruction function $r(x)$ (in turquoise)
which would be learned by a high-capacity auto-encoder
on a 1-dimensional input, i.e., minimizing
reconstruction error {\em at the training examples} $x_i$
(with $r(x_i)$ in red) while
trying to be as constant as possible otherwise.
The figure is used to exagerate and illustrate the effect of the regularizer
(corresponding to a large $\sigma^2$
in the loss function $\mathcal{L}$ later described by (\ref{eqn:RCAE_reconstruction_loss})).
The dotted line is
the identity reconstruction (which might be obtained without the regularizer).
The blue arrows shows the vector field of $r(x)-x$ pointing towards
high density peaks as estimated by the model, and estimating
the score (log-density derivative), as shown in this paper.
} \label{fig:1D-autoencoder}
\vspace*{-3mm}
\end{figure}


Here we show that the denoising auto-encoder~\citep{VincentPLarochelleH2008} with very small Gaussian
corruption and squared error loss is actually a particular kind of
contractive auto-encoder~\citep{Rifai+al-2011}, contracting the whole auto-encoder reconstruction function
rather than just the encoder, whose contraction penalty coefficient is the magnitude
of the perturbation. This was first suggested in~\citep{Rifai+al-2011-small}.


The contractive auto-encoder, or CAE~\citep{Rifai+al-2011}, is a particular
form of regularized auto-encoder which
is trained to minimize the following regularized
reconstruction error:
\begin{equation}
  {\cal L}_{CAE} = \E\left[ \ell(x,r(x))  + \lambda \left\|\frac{\partial f(x)}{\partial x}\right\|^2_F \right]
\label{eq:cae-crit}
\end{equation}
where $r(x)=g(f(x))$ and $||A||^2_F$ is the sum of the squares of the elements of $A$.
Both the squared loss $\ell(x,r)=||x-r||^2$ and the
cross-entropy loss $\ell(x,r)=-x\log r -(1-x)\log(1-r)$ have been
used, but here we focus our analysis on the squared loss because
of the easier mathematical treatment it allows. Note that success in 
minimizing the CAE criterion
strongly depends on the parametrization of $f$ and $g$ and in particular
on the tied weights constraint used, with $f(x)=\sigm(W x + b)$
and $g(h)=\sigm(W^T h + c)$. The above regularizing term forces $f$ (as well
as $g$, because of the tied weights) to be contractive, i.e.,
to have singular values less than 1~\footnote{Note that an auto-encoder
without any regularization would tend to find many leading singular values near 1 in
order to minimize reconstruction error, i.e., preserve input norm
in all the directions of variation present in the data.}. Larger values of $\lambda$
yield more contraction (smaller singular values) where it hurts reconstruction
error the least, i.e., in the local directions where there are only little or no
variations in the data. These typically are the directions orthogonal to the
manifold of high density concentration, as illustrated in Fig.~\ref{fig:1D-autoencoder}.

\vspace{1.5em}

The denoising auto-encoder, or DAE~\citep{VincentPLarochelleH2008},
is trained to minimize the following denoising criterion:
\begin{equation}
 {\cal L}_{DAE} = \E\left[ \ell(x,r(N(x))) \right]
\label{eq:dae-crit}
\end{equation}
where $N(x)$ is a stochastic corruption of $x$ and the expectation is over
the training distribution and the corruption noise source. Here we consider mostly
the squared loss and Gaussian noise corruption, again because it is easier to handle
them mathematically. In many cases, the exact same proofs can be applied to any
kind of additive noise, but Gaussian noise serves as a good frame of reference.

\begin{theorem}
\label{thm:DAE-optimal-solution}
Let $p$ be the probability density function of the data.
If we train a $DAE$ using the expected quadratic loss
and corruption noise $N(x) = x+\epsilon$ with
\[
\epsilon \sim \mathcal{N}\left(0, \sigma^2 I \right),
\]
then the optimal reconstruction function $r^{*}(x)$ will be given by
\begin{equation}
r^*(x) = \frac{ \mathbb{E}_{\epsilon} \left[ p(x-\epsilon) (x - \epsilon) \right]}{\mathbb{E}_{\epsilon} \left[ p(x-\epsilon) \right]} \label{eqn:opt-ratio}
\end{equation}
for values of $x$ where $p(x) \neq 0$.

Moreover, if we consider how the optimal reconstruction function
$r^{*}_{\sigma}(x)$ behaves asymptotically as $\sigma \rightarrow 0$,
we get that
\begin{equation}
r^*_\sigma(x) = x + \sigma^2 \frac{\partial \log p(x)}{\partial x} + o(\sigma^2) \hspace{1em} \textrm{as} \hspace{1em} \sigma \rightarrow 0. \label{eqn:opt-asympt}
\end{equation}
\end{theorem}

The proof of this result is found in the Appendix.
We make use of the small $o$ notation throughout this paper and
assume that the reader is familiar with asymptotic notation.
In the context of Theorem \ref{thm:DAE-optimal-solution}, it has to be understood
that all the other quantities except for $\sigma$ are fixed when
we study the effect of $\sigma \rightarrow 0$.
Also, note that the $\sigma$ in the index of $r_\sigma^*$ is to indicate
that $r_\sigma^*$ was chosen based on the value of $\sigma$.
That $\sigma$ should not be mistaken for a parameter to be learned.

Equation (\ref{eqn:opt-ratio}) reveals that the optimal DAE reconstruction function
at every point $x$ is given by a kind of convolution
involving the density function $p$, or weighted average from the points
in the neighbourhood of $x$, depending on how we would like to view it.
A higher noise level $\sigma$ means that a larger neighbourhood of $x$
is taken into account. Note that the total quantity of ``mass'' being
included in the weighted average of the numerator of (\ref{eqn:opt-ratio})
is found again at the denominator.

Gaussian noise is a simple case in the sense that it is additive and symmetrical, so
it avoids the complications that would occur when trying to integrate over the density of pre-images $x'$
such that $N(x')=x$ for a given $x$.
The ratio of those quantities that we have in equation (\ref{eqn:opt-ratio}), however, depends
strongly on the decision that we made to minimize the expected square error.

\vspace{1.5em}

When we look at the asymptotic behavior with equation (\ref{eqn:opt-asympt}),
the first thing to observe is that the leading term in the expansion of $r^*_\sigma(x)$ is $x$,
and then the remainder goes to $0$ as $\sigma \rightarrow 0$. When there is no noise left at all,
it should be clear that the best reconstruction target for any value $x$ would be that $x$ itself.

We get something even more interesting if we look at the second term of equation (\ref{eqn:opt-asympt})
because it gives us an estimator of the score from
\begin{equation}
\frac{\partial\log p(x)}{\partial x} = \left(r_\sigma^*(x)- x \right)/ {\sigma^2} + o(1) \hspace{1em} \textrm{as} \hspace{1em} \sigma \rightarrow 0.
\label{eqn:score-estimator}
\end{equation}
This result is at the core of our paper. It is what allows us to start from a 
trained DAE, and then recover properties of the training density $p(x)$
that can be used to sample from $p(x)$.

Most of the asymptotic properties that we get by considering
the limit as the Gaussian noise level $\sigma$ goes to $0$ could be derived from a
family of noise distribution that approaches a point mass distribution in a
relatively ``nice'' way.

\vspace{1.5em}

An interesting connection with contractive auto-encoders can also be observed
by using a Taylor expansion of the denoising auto-encoder loss and assuming only that
$r_\sigma(x)=x+o(1)$ as $\sigma \rightarrow 0$. In that case we get the following proposition.

\setcounter{theorem}{0}
\begin{proposition}
\label{prp:DAE-connection-RCAE}
Let $p$ be the probability density function of the data.
Consider a $DAE$ using the expected quadratic loss
and corruption noise $N(x) = x+\epsilon$, with $\epsilon \sim \mathcal{N}\left(0, \sigma^2 I \right)$.
If we assume that the non-parametric solutions $r_\sigma(x)$ satistfies
\[
r_\sigma(x)=x+o(1) \hspace{1em} \textrm{as} \hspace{1em} \sigma \rightarrow 0,
\]
then we can rewrite the loss as 
\[
\mathcal{L}_{\textrm{DAE }}(r_\sigma) = \E\left[\|r_\sigma(x) - x\|^2_2 + \sigma^2 \left\|\frac{\partial r_\sigma(x)}{\partial x}\right\|^2_F\right]
+ o(\sigma^2) \hspace{1em} \textrm{as} \hspace{1em} \sigma \rightarrow 0
\]
where the expectation is taken with respect to $X$, whose distribution is given by $p$.
\end{proposition}
The proof is in Appendix and uses a simple Taylor expansion around $x$.

Proposition \ref{prp:DAE-connection-RCAE} shows that {\em the DAE with small corruption of variance $\sigma^2$
is similar to a contractive auto-encoder with penalty coefficient $\lambda=\sigma^2$} but
where the contraction is imposed explicitly on the whole reconstruction function
$r(\cdot)=g(f(\cdot))$ rather than on $f(\cdot)$ alone\footnote{In the CAE there is a also
a contractive effect on $g(\cdot)$ as a side
effect of the parametrization with weights tied between $f(\cdot)$ and $g(\cdot)$.}.

This analysis motivates the definition of the \emph{reconstruction contractive auto-encoder} (RCAE),
a variation of the CAE where loss function is instead the squared reconstruction loss
plus contractive penalty on the reconstruction:
\begin{equation}
\label{eqn:RCAE_reconstruction_loss}
{\cal L}_{\textrm{RCAE}} = \E\left[\|r(x) - x\|^2_2 + \sigma^2 \left\|\frac{\partial r(x)}{\partial x}\right\|^2_F\right].
\end{equation}
This is an analytic version of the denoising criterion with small noise $\sigma^2$,
and also corresponds to a contractive auto-encoder with contraction on both $f$ and $g$, i.e., on $r$.

Because of the similarity between DAE and RCAE when taking
$\lambda = \sigma^2$ and because the semantics of $\sigma^2$ is clearer
(as a squared distance in input space), we will denote $\sigma^2$ for the penalty term coefficient
in situations involving RCAE.
For example, in the statement of Theorem \ref{thm:calcvarminloss},
this $\sigma^2$ is just a positive constant;
there is no notion of additive Gaussian noise,
i.e., $\sigma^2$ does not explicitly refer to a variance,
but using the notation $\sigma^2$ makes it easier to
intuitively see the connection to the DAE setting.

The connection between DAE and RCAE established in Proposition \ref{prp:DAE-connection-RCAE}
motivates the following Theorem \ref{thm:calcvarminloss} as an alternative
way to achieve a result similar to that of Theorem \ref{thm:DAE-optimal-solution}.
In this theorem we study the asymptotic behavior of the RCAE solution.

\begin{theorem}
\label{thm:calcvarminloss} 

Let $p$ be a probability density function that is continuously differentiable
once and with support $\mathbb{R}^{d}$ (i.e. $\forall x\in\mathbb{R}^{d}$
we have $p(x)\neq0$). Let $\mathcal{L}_{\sigma}$ be the loss function defined
by
\begin{equation}\label{eqn:loss_r_in_thm}
\mathcal{L}_{\sigma}(r)=\int_{\mathbb{R}^{d}}p(x)\left[\left\Vert r(x)-x\right\Vert _{2}^{2}+{\sigma^2}\left\Vert \frac{\partial r(x)}{\partial x}\right\Vert _{F}^{2}\right]dx
\end{equation}
for $r:\mathbb{R}^{d}\rightarrow\mathbb{R}^{d}$ assumed to be
differentiable twice, and $0 \leq {\sigma^2}\in\mathbb{R}$ used as factor
to the penalty term.

Let $r_{\sigma}^*(x)$ denote the optimal function that minimizes
$\mathcal{L}_{\sigma}$. Then we have that

\begin{equation}
r_{\sigma}^*(x) = x + {\sigma^2}\frac{\partial\log p(x)}{\partial x}+o({\sigma^2})\hspace{1em}\textrm{as}\hspace{1em}{\sigma} \rightarrow 0.\nonumber
\end{equation}

Moreover, we also have the following expression for the derivative

\begin{equation}\label{eq:calcvarminloss-drx}
\frac{\partial r_{\sigma}^*(x)}{\partial x} = I + {\sigma^2}\frac{\partial^2\log p(x)}{\partial x^2} + o({\sigma^2})\hspace{1em}\textrm{as}\hspace{1em}{\sigma} \rightarrow 0.
\end{equation}

Both these asymptotic expansions are to be understood
in a context where we consider $\left\{r_{\sigma}^*(x)\right\}_{{\sigma} \geq 0}$
to be a family of optimal functions minimizing $\mathcal{L}_{\sigma}$
for their corresponding value of ${\sigma}$.
The asymptotic expansions are applicable point-wise in $x$, that is,
with any fixed $x$ we look at the behavior as ${\sigma} \rightarrow 0$.
\end{theorem}

The proof is given in the appendix and uses the Euler-Lagrange
equations from the calculus of variations.

\section{Minimizing the Loss to Recover Local Features of $p(\cdot)$}
\label{seq:core-discussion}

One of the central ideas of this paper is that in a non-parametric setting
(without parametric constraints on $r$), 
we have an asymptotic formula (as the noise level ${\sigma} \rightarrow 0$)
for the optimal reconstruction function for the DAE and RCAE
that allows us to recover the score $\frac{\partial \log p(x)}{\partial x}$.

A DAE is trained with a method that knows nothing about $p$,
except through the use of training samples to minimize a loss function,
so it comes as a surprise that we can compute the score of $p$ at
any point $x$.

In the following subsections we explore the consequences and the
practical aspect of this.

\subsection{Empirical Loss}

In an experimental setting, the expected loss (\ref{eqn:loss_r_in_thm}) is
replaced by the empirical loss

\[
\hat{\mathcal{L}} = \frac{1}{N} \sum_{n=1}^N
   \left(
      \left\Vert r(x^{(n)}) - x^{(n)} \right\Vert _2^2
      + {\sigma^2} \left\Vert
                     \left.
                          \frac{\partial r(x)}{\partial x}
                     \right|_{x = x^{(n)}}
                \right\Vert _F^2
   \right)
   \label{eqn:empirical-loss}
\]

based on a sample $\left\{ x^{(n)} \right\}_{n=1}^N$ drawn from $p(x)$.

Alternatively, the auto-encoder is trained online (by stochastic
gradient updates) with a stream
of examples $x^{(n)}$, which corresponds to performing stochastic
gradient descent on the expected loss (\ref{eqn:loss_r_in_thm}). In both
cases we obtain an auto-encoder that approximately minimizes
the expected loss.

An interesting question is the following: what can we infer
from the data generating density when given an auto-encoder
reconstruction function $r(x)$?

The premise is that this auto-encoder $r(x)$ was
trained to approximately minimize a loss function that has exactly
the form of (\ref{eqn:loss_r_in_thm}) for some ${\sigma^2}>0$. This is assumed to have 
been done through minimizing the empirical loss and the distribution
$p$ was only available indirectly through the samples 
$\left\{ x^{(n)} \right\}_{n=1}^N$. We do not have access to $p$
or to the samples. We have only $r(x)$ and maybe ${\sigma^2}$.

We will now discuss the usefulness of $r(x)$ based on different
conditions such as the model capacity and the value of ${\sigma^2}$.

\subsection{Perfect World Scenario}
\label{sec:perfect-world-scenario}

As a starting point, we will assume that we are in a perfect
situation, i.e., with no constraint on $r$ (non-parametric setting),
an infinite amount of training data, and a perfect minimization. 
We will see what can be done to recover information about $p$ in that ideal case.
Afterwards, we will drop certain assumptions one by one and
discuss the possible paths to getting back some information about $p$.

We use notation $r_{\sigma}(x)$ when we want to
emphasize the fact that the value of $r(x)$ came from minimizing
the loss with a certain fixed ${\sigma}$.

Suppose that $r_{\sigma}(x)$ was trained with an infinite sample
drawn from $p$.
Suppose also that it had infinite (or sufficient) model capacity
and that it is capable of achieving the minimum of the
loss function (\ref{eqn:loss_r_in_thm}) while satisfying the
requirement that $r(x)$ be twice differentiable.
Suppose that we know the value of ${\sigma}$ and that
we are working in a computing environment of arbitrary precision
(i.e. no rounding errors).

As shown by Theorem \ref{thm:DAE-optimal-solution} and Theorem \ref{thm:calcvarminloss}, we would be able
to get numerically the values of
$\frac{\partial \log p(x)}{\partial x}$ at any point $x\in\mathbb{R}^d$
by simply evaluating
\begin{equation}
    \frac{r_{\sigma}(x) - x}{{\sigma^2}} \rightarrow \frac{\partial \log p(x)}{\partial x}
    \hspace{1em} \textrm{as} \hspace{1em} {\sigma} \rightarrow 0. \label{eqn:rx-x-trick}
\end{equation}
In the setup described, we do not get to pick values of
${\sigma}$ so as to take the limit ${\sigma} \rightarrow 0$.
However, it is assumed that ${\sigma}$ is already sufficiently
small that the above quantity is close to
$\frac{\partial \log p(x)}{\partial x}$ for all intents
and purposes.

\subsection{Simple Numerical Example}
\label{sec:simple-numerical-example}

To give an example of this in one dimension,
we will show what happens when we train a non-parametric
model $\hat{r}(x)$ to minimize numerically the loss relative to $p(x)$.
We train both a DAE and an RCAE in this fashion by minimizing
a discretized version of their losses defined by equations
(\ref{eq:dae-crit}) and (\ref{eqn:RCAE_reconstruction_loss}).
The goal here is to show that, for either a DAE or RCAE, the approximation of the score
that we get through equation (\ref{eqn:score-estimator})
gets arbitrarily close to the actual score
$\frac{\partial}{\partial x} \log p(x)$ as $\sigma \rightarrow 0$.

The distribution $p(x)$ studied is shown in Figure \ref{fig:one-dim-reconst2} (left)
and it was created to be simple enough to illustrate the
mechanics. We plot $p(x)$ in Figure \ref{fig:one-dim-reconst2} (left) along
with the score of $p(x)$ (right).

\begin{figure}[ht]
\centering
\mbox{
    \subfigure[$p(x)=\frac{1}{Z}\exp(-E(x))$]{
      \includegraphics[scale=0.35]{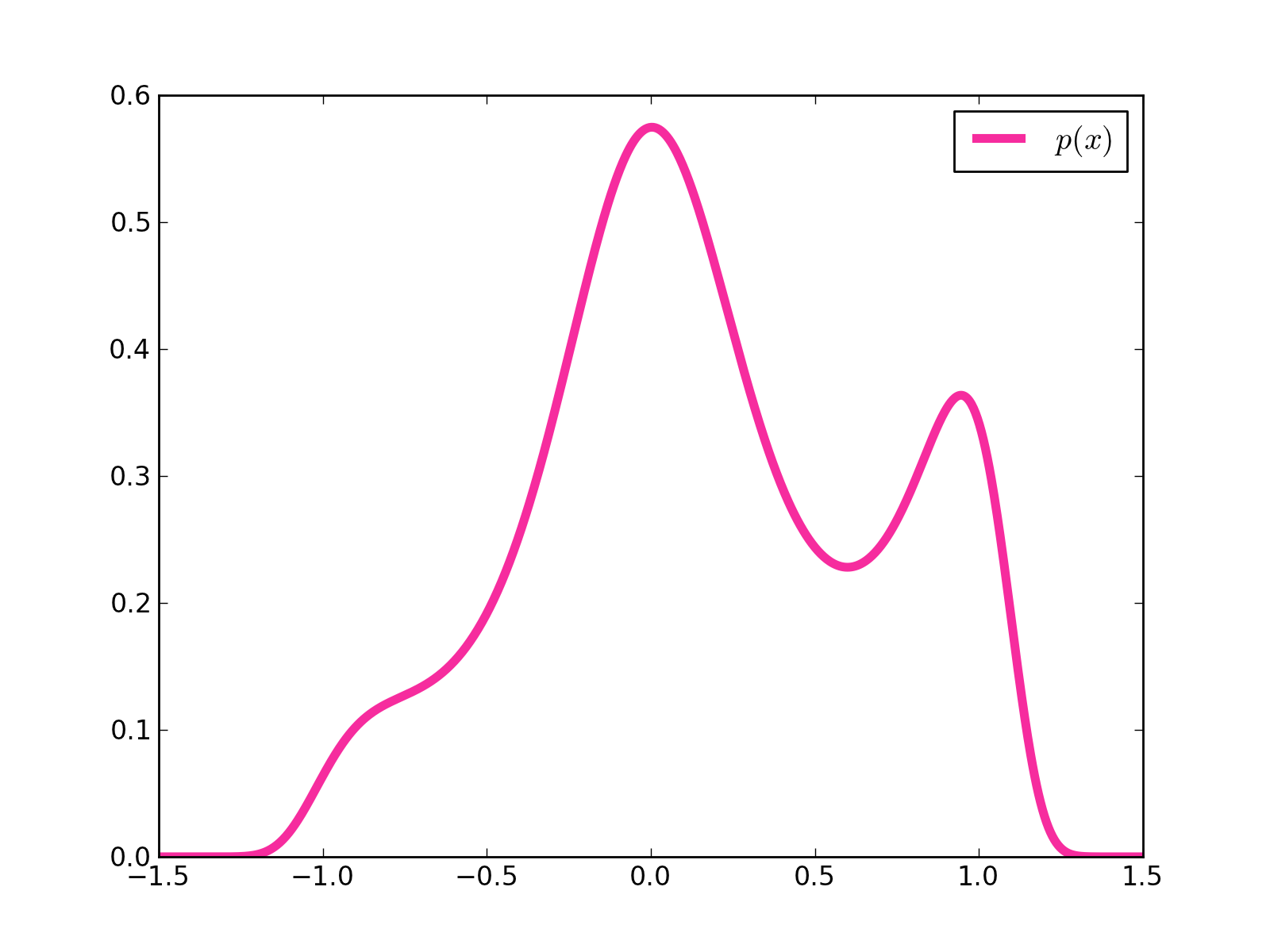}
      \label{fig:one-dim-pdf}
    }
    \quad
    \subfigure[$\frac{\partial}{\partial x} \log p(x) = - \frac{\partial}{\partial x} E(x)$]{
      \includegraphics[scale=0.35]{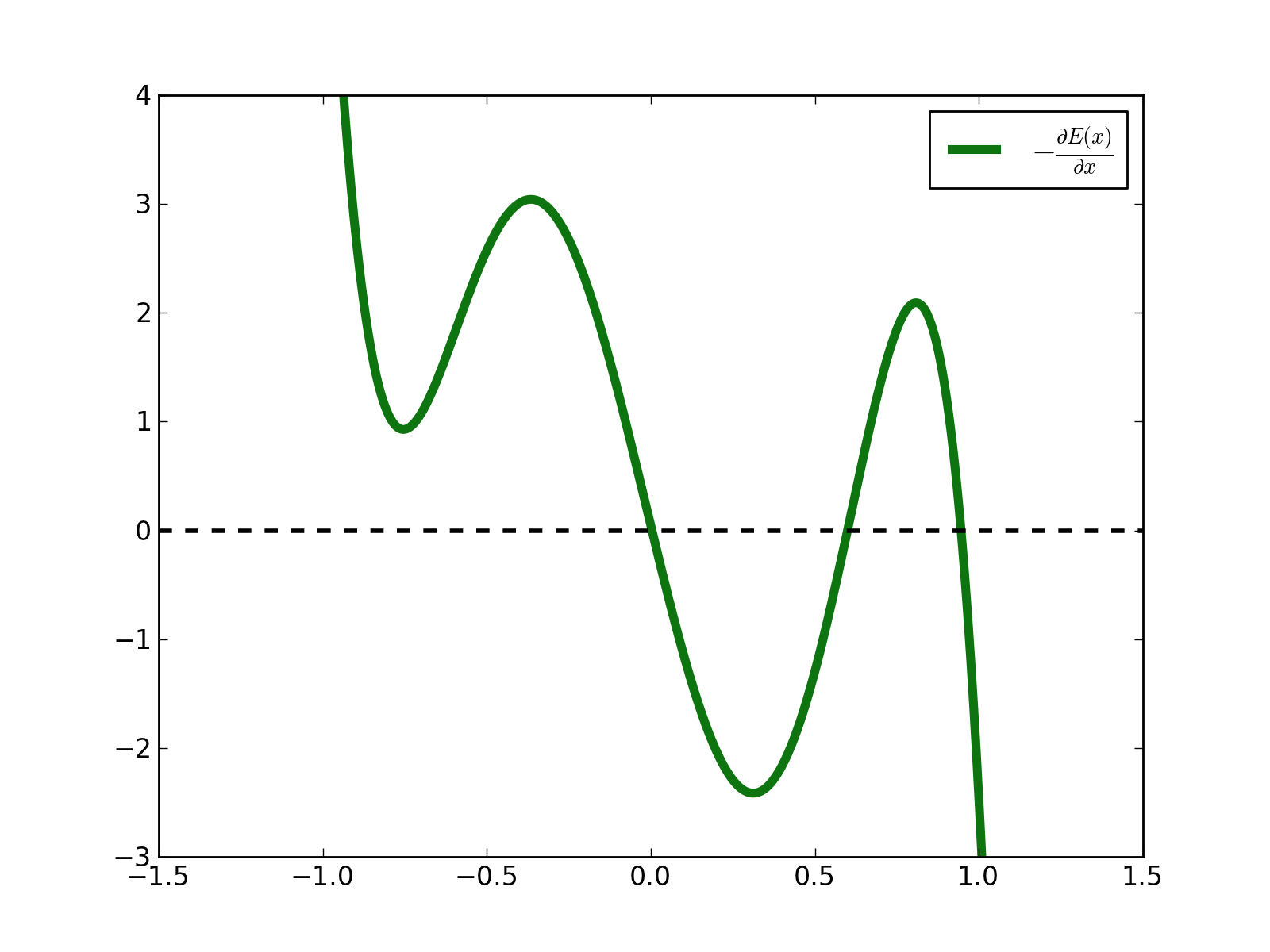}
      \label{fig:one-dim-grad-E}
    }
}
\caption{the density $p(x)$ and its score for a simple one-dimensional example.}
\label{fig:one-dim-reconst2}
\end{figure}

The model $\hat{r}(x)$ is fitted by dividing the interval $\left[ -1.5,1.5 \right]$
into $M=1000$ partition points $x_1,\ldots,x_M$ evenly separated by a distance $\Delta$.
The discretized version of the RCAE loss function is
\begin{equation}
\sum_{i=1}^M p(x_i) \Delta \left( \hat{r}(x_i) - x_i \right)^2
+ {\sigma^2} \sum_{i=1}^{M-1} p(x_i) \Delta \left( \frac{\hat{r}(x_{i+1}) - \hat{r}(x_i)}{\Delta} \right)^2.\label{eqn:loss_r_in_thm_r_numerical_partition}
\end{equation}
Every value $\hat{r}(x_i)$ for $i=1,\ldots,M$ is treated as a free parameter.
Setting to $0$ the derivative with respect to the $\hat{r}(x_i)$ yields
a system of linear equations in $M$ unknowns that we can solve exactly.
From that RCAE solution $\hat{r}$ we get an approximation of the score of $p$
at each point $x_i$. A similar thing can be done for the DAE by using
a discrete version of the exact solution (\ref{eqn:opt-ratio})
from Theorem \ref{thm:DAE-optimal-solution}.
We now have two ways of approximating the score of $p$.

In Figure \ref{fig:score-movie-frames} we compare the approximations
to the actual score of $p$ for decreasingly small values of $\sigma \in \{ 1.00, 0.31, 0.16, 0.06\}$.

\begin{figure}[htb]
\centering
\includegraphics[width=0.45\textwidth]{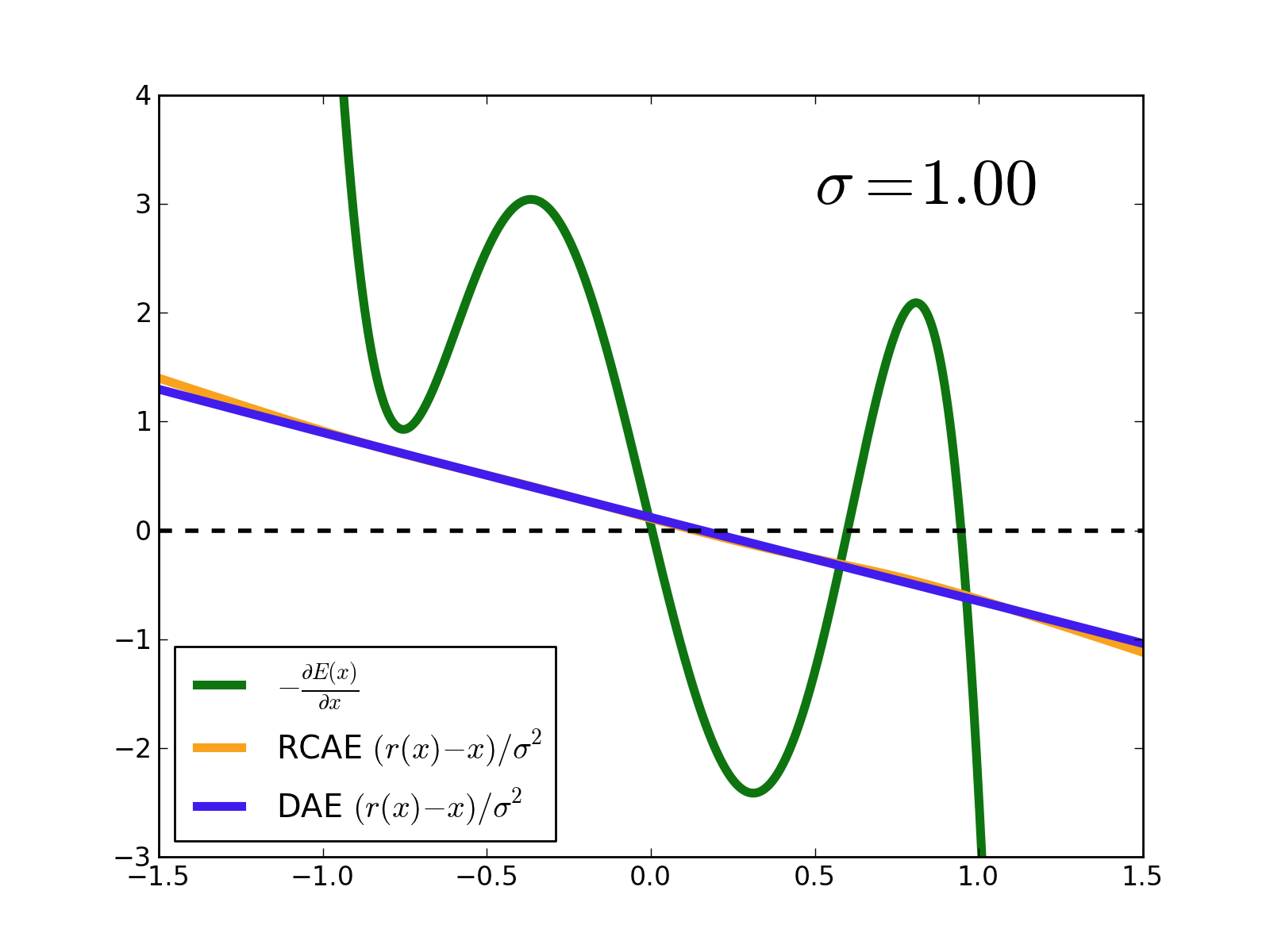}
\includegraphics[width=0.45\textwidth]{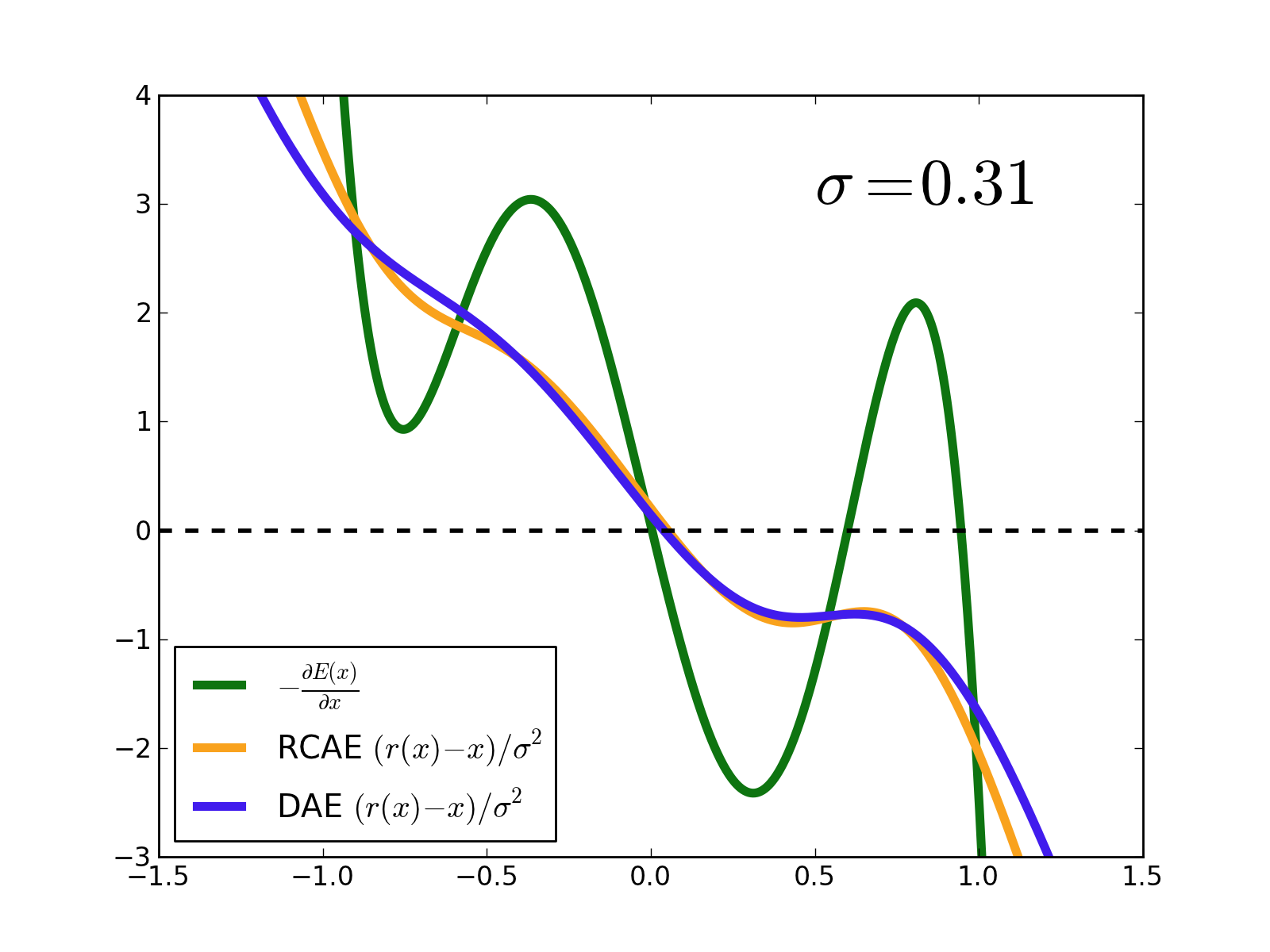}
\includegraphics[width=0.45\textwidth]{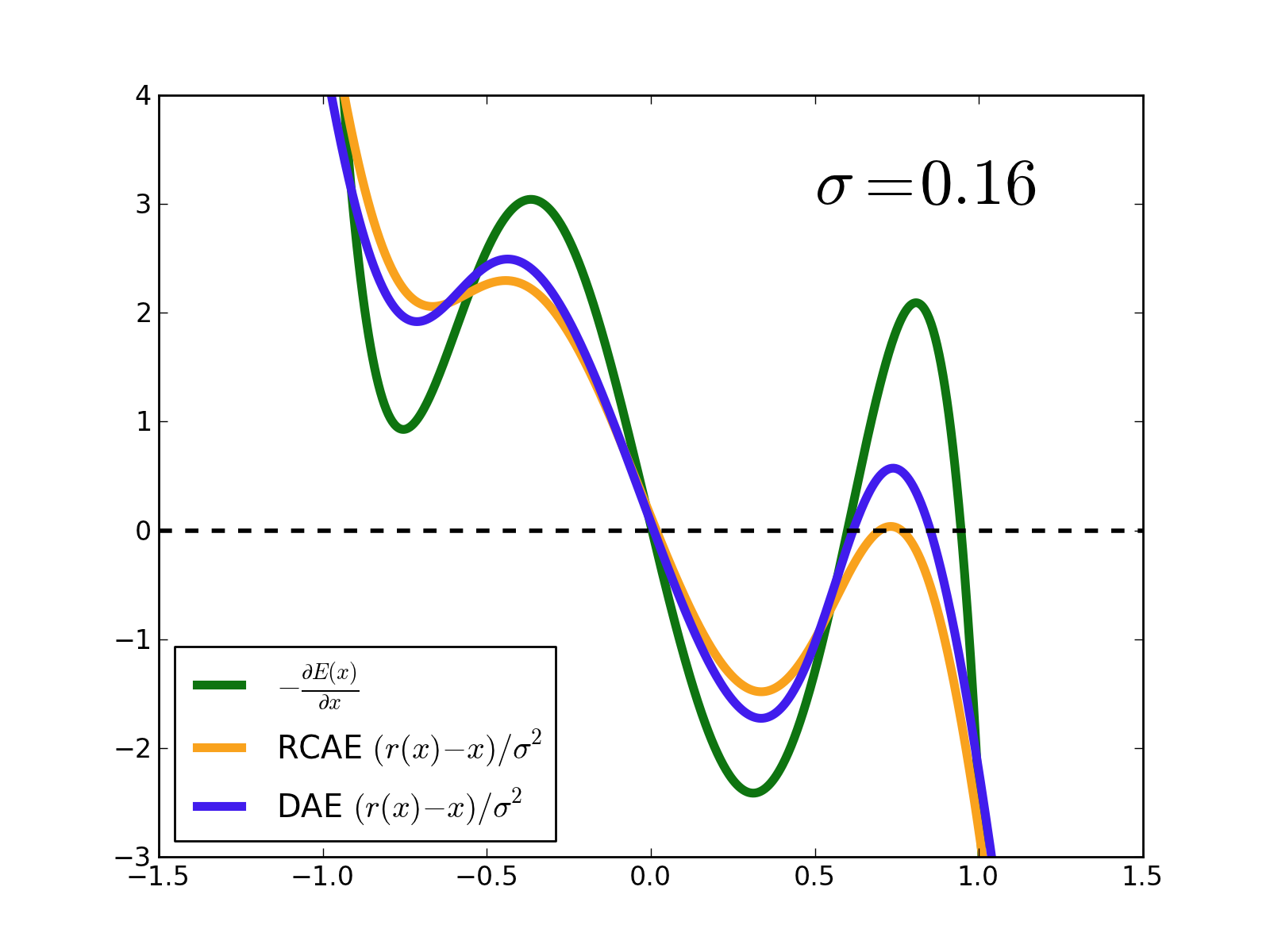}
\includegraphics[width=0.45\textwidth]{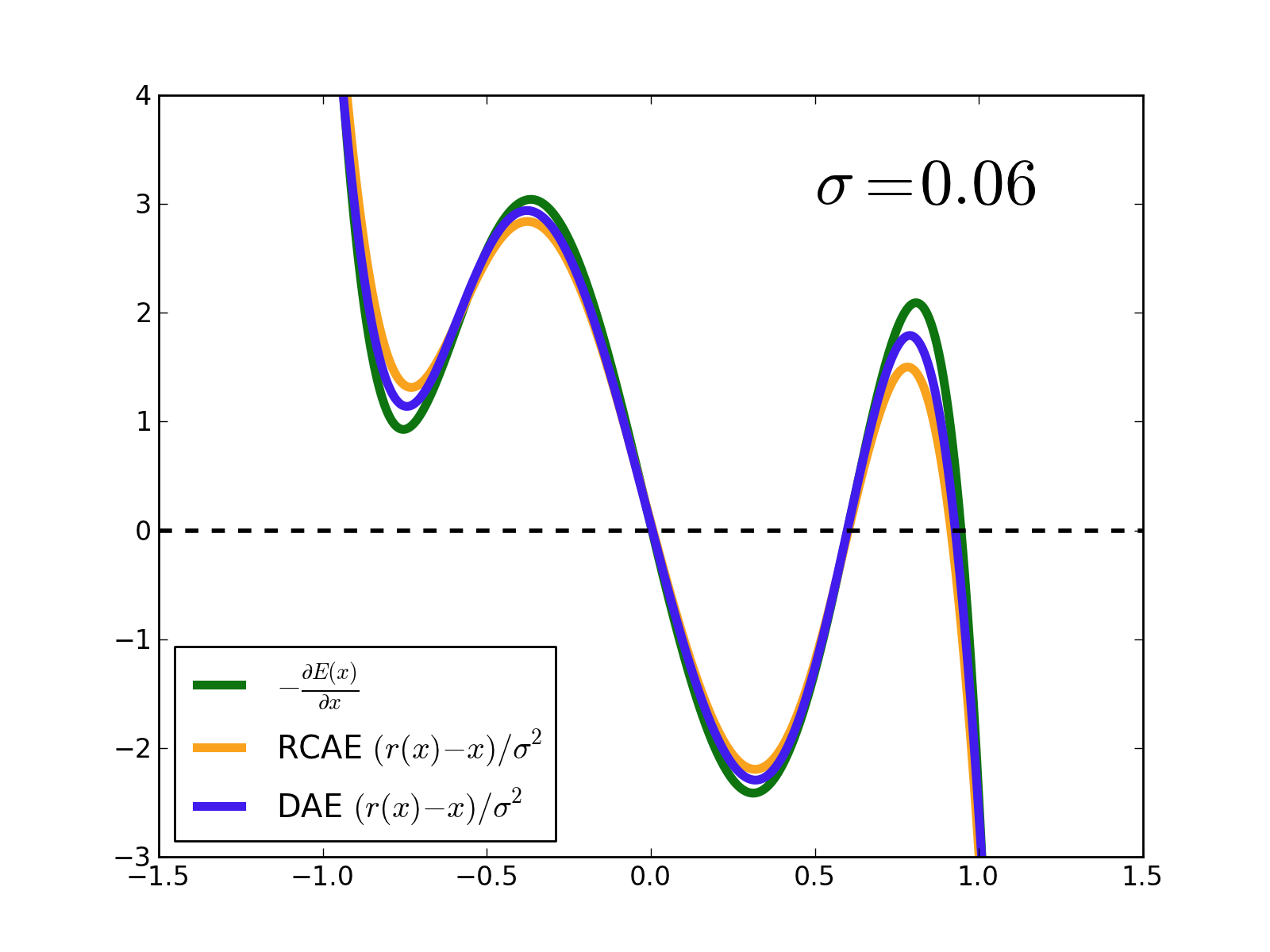}
\caption{Comparing the approximation of the score of $p$ given by discrete versions of
optimally trained auto-encoders with infinite capacity. The approximations given by the RCAE
are in orange while the approximations given by the DAE are in purple. The results are shown
for decreasing values of $\sigma \in \{1.00, 0.31, 0.16, 0.06\}$ that have been selected for
their visual appeal.
\newline
As expected, we see in that the RCAE (orange)
and DAE (purple) approximations of the score are close to each other
as predicted by Proposition \ref{prp:DAE-connection-RCAE}. Moreover, they are also
converging to the true score (green)
as predicted by Theorem \ref{thm:DAE-optimal-solution} and Theorem \ref{thm:calcvarminloss}.
}
\label{fig:score-movie-frames}
\end{figure}

\subsection{Vector Field Around a Manifold}

We extend the experimentation of section \ref{sec:simple-numerical-example}
to a 1-dimensional manifold in 2-D space, in which one can
visualize $r(x)-x$ as a vector field, and we go from the non-parametric estimator of the previous
section to an actual auto-encoder trained by numerically minimizing the regularized
reconstruction error.

Two-dimensional data points $(x,y)$ were generated along a spiral according to the following
equations:

\[
x = 0.04 \sin(t) , \hspace{1em} y = 0.04 \cos(t), \hspace{1em} t \sim \textrm{Uniform}\left(3,12 \right).
\]

A denoising auto-encoder was trained with Gaussian corruption noise $\sigma=0.01$.
The encoder is $f(x)=\tanh(b + Wx)$ and the decoder is $g(h)=c+Vh$. The parameters
$(b,c,V,W)$ are optimized by BFGS to minimize the average squared error, 
using a fixed training set of $10\ 000$ samples
(i.e. the same corruption noises were sampled once and for all). We found better
results with untied weights, and BFGS gave more accurate models than stochastic
gradient descent. We used $1000$ hiddens units and ran BFGS for 1000 iterations.

The non-convexity of the problem makes it such that the solution found depends
on the initialization parameters. The random corruption noise used
can also influence the final outcome. Moreover, the fact that we are using a
finite training sample size with reasonably small noise may allow for
undesirable behavior of $r$ in regions far away from the training samples.
For those reasons, we trained the model multiple times and selected two of the
most visually appealing outcomes. These are found in Figure~\ref{fig:two-spiral-graphics}
which features a more global perspective along with a close-up view.

\begin{figure}[h]
\centering
\mbox{
    \subfigure[][$r(x)-x$ vector field, acting as sink, zoomed out]{
      \includegraphics[scale=0.36]{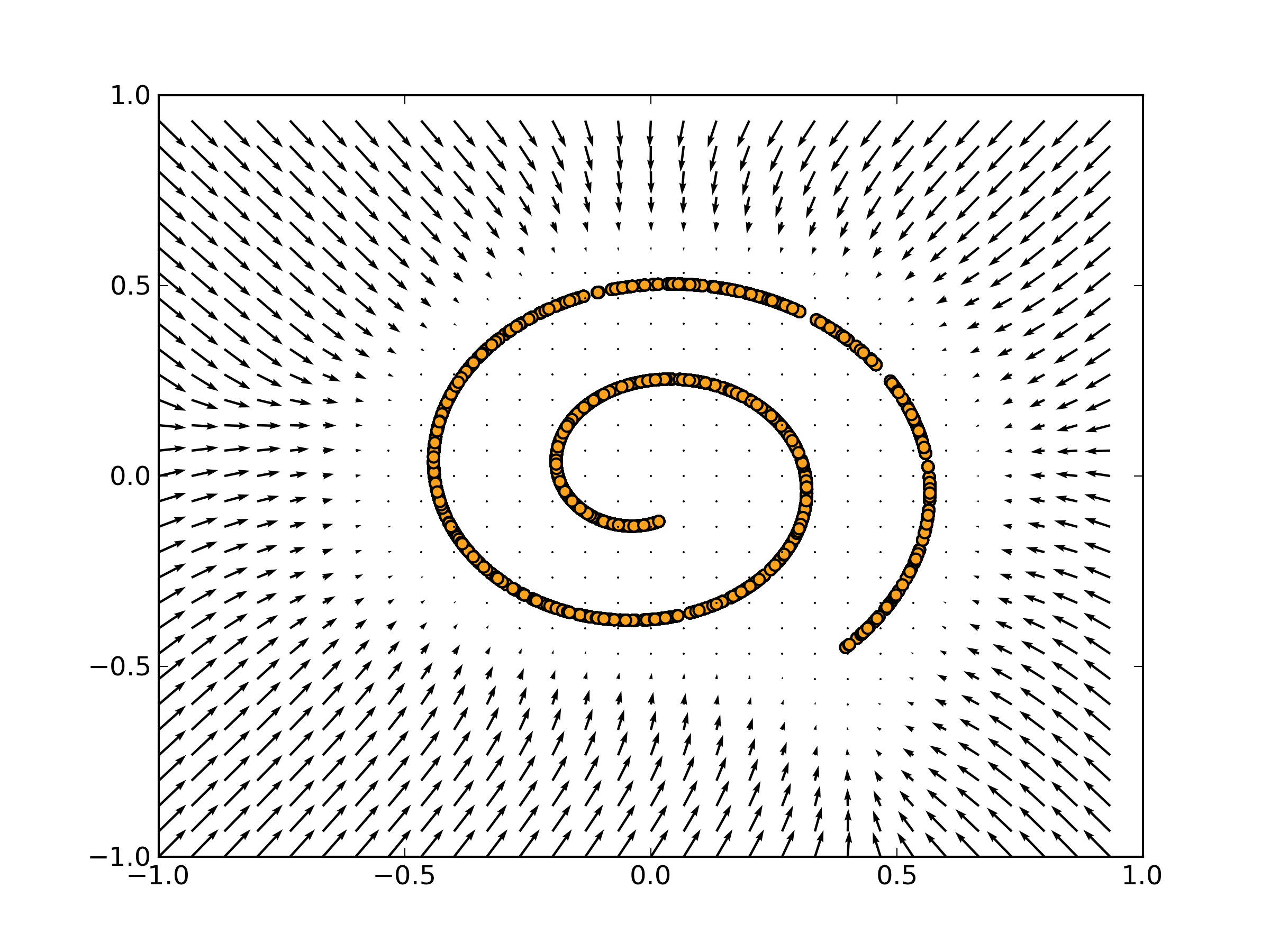}
      \label{fig1:spiral-reconstruction-grid-full}
    }
    \quad
    \subfigure[][$r(x)-x$ vector field, close-up]{
      \includegraphics[scale=0.36]{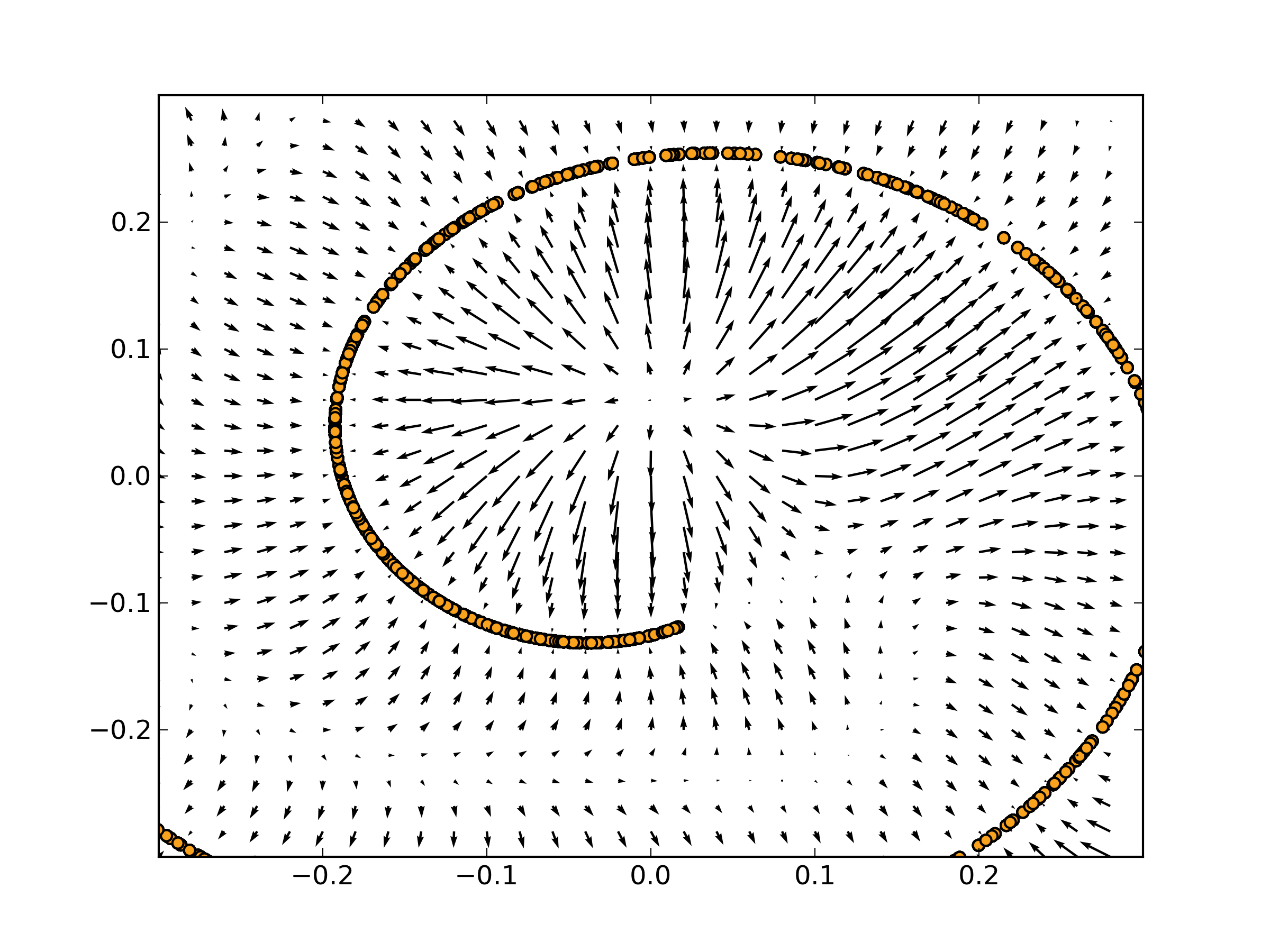}
      \label{fig2:spiral-reconstruction-grid-zoomed-center}
    }
}
\caption{The original 2-D data from the data generating density $p(x)$ is plotted along
with the vector field defined by the values of $r(x)-x$
for trained auto-encoders (corresponding to the estimation of the score $\frac{\partial \log p(x)}{\partial x}$).}
\label{fig:two-spiral-graphics}
\end{figure}

Figure~\ref{fig:two-spiral-graphics} shows the data along with the learned score function
(shown as a vector field).
We see that that the vector field points
towards the nearest high-density point on the data manifold.
The vector field is close to zero near the manifold (i.e. the
reconstruction error is close to zero), also corresponding to peaks
of the implicitly estimated density. The points on the manifolds play the role
of sinks for the vector field. {\em Other places where reconstruction error may be low,
but where the implicit density is not high, are sources of the vector field.}
In Figure~\ref{fig2:spiral-reconstruction-grid-zoomed-center} we can see
that we have that kind of behavior halfway between two sections
of the manifold. This shows that reconstruction error plays a very different
role as what was previously hypothesized: whereas~\citet{ranzato-08} viewed
reconstruction error as an {\em energy function}, our analysis suggests
that in regularized auto-encoders, it is the norm of an approximate score,
i.e., the derivative of the energy w.r.t. input. Note that the norm of the
score should be small near training examples (corresponding to local maxima of
density) but it could also be small at other places corresponding to
{\em local minima of density}. This is indeed what happens in the spiral
example shown. It may happen whenever there are high-density regions
separated by a low-density region: tracing paths from one high-density
region to another should cross a ``median'' lower-dimensional region (a manifold)
where the density has a local maximum along the path direction. The reason
such a median region is needed is because at these points the vectors $r(x)-x$
must change sign: on one side of the median they point to one of the high-density
regions while on the other side they point to the other, as clearly 
visible in Figure~\ref{fig2:spiral-reconstruction-grid-zoomed-center} 
between the arms of the spiral. 

We believe that this analysis is valid not just for contractive and
denoising auto-encoders, but for regularized auto-encoders in general.
The intuition behind this statement can be firmed up by analyzing
Figure~\ref{fig:1D-autoencoder}: the score-like behavior of $r(x)-x$
arises simply out of the opposing forces of (a) trying to make $r(x)=x$
at the training examples and (b) trying to make $r(x)$ as regularized
as possible (as close to a constant as possible). 

Note that previous work~\citep{Rifai-icml2012,Bengio-et-al-icml2013} has already shown that contractive auto-encoders
(especially when they are stacked in a way similar to RBMs in a deep belief net) 
learn good models of high-dimensional data (such as images), and that these models can
be used not just to obtain good representations for classification tasks but
that good quality samples can be obtained from the model, by a random walk
near the manifold of high-density. This was achieved by essentially following the 
vector field and adding noise along the way.

\subsection{Missing ${\sigma^2}$}

When we are in the same setting as in section \ref{sec:perfect-world-scenario}
but the value of ${\sigma^2}$ is unknown, we can modify
(\ref{eqn:rx-x-trick}) a bit and avoid dividing by ${\sigma^2}$.
That is, for a trained reconstruction function $r(x)$
given to us we just
take the quantity $r(x)-x$ and it should be approximatively
the score {\em up to a multiplicative constant}.
\[
    r(x)-x \propto \frac{\partial \log p(x)}{\partial x}
\]
Equivalently, if one estimates the density via an energy
function (minus the unnormalized log density), then $x-r(x)$
estimates the derivative of the energy function.

We still have to assume that ${\sigma^2}$ is small.
Otherwise, if the unknown ${\sigma^2}$ is too large we
might get a poor estimation of the score.

\subsection{Limited Parameterization}
\label{sec:limited-parameterization}

We should also be concerned about the fact that $r(x)-x$
is trying to approximate $- \frac{\partial E(x)}{\partial x}$
as ${\sigma} \rightarrow 0$ but we have not made any assumptions
about the space of functions that $r$ can represent when we
are dealing with a specific implementation.

When using a certain parameterization of $r$ such as the one
from section \ref{sec:simple-numerical-example}, there is
no guarantee that the family of functions in which we select $r$
each represent a conservative vector field (i.e. the gradient
of a potential function). Even if we start
from a density $p(x) \propto \exp(-E(x))$ and we
have that $r(x)-x$ is very close to $- \frac{\partial}{\partial x} E(x)$
in terms of some given norm, there is not guarantee that there
exists an associated function $E_0(x)$
for which $r(x)-x \propto - \frac{\partial }{\partial x} E_0(x)$ and $E_0(x) \approx E(x)$.

In fact, in many cases we can trivially show the non-existence of
such a $E_0(x)$ by computing the curl of $r(x)$.
The curl has to be equal to $0$ everywhere if $r(x)$ is indeed the
derivative of a potential function. We can omit the $x$
terms from the computations because we can easily find its antiderivative
by looking at $x = \frac{\partial}{\partial x} \left\| x \right\| ^2_2$.

Conceptually, another way to see this is to argue that
if such a function $E_0(x)$ existed,
its second-order mixed derivatives should be equal. That is,
we should have that
\[
\frac{\partial^2 E_0(x)}{\partial x_i \partial x_j} = \frac{\partial^2 E_0(x)}{\partial x_j \partial x_i} \hspace{1em} \forall i,j,
\]
which is equivalent to
\[
\frac{\partial r_i(x)}{\partial x_j} = \frac{\partial r_j(x)}{\partial x_i} \hspace{1em} \forall i,j.
\]

Again in the context of section \ref{sec:simple-numerical-example},
with the parameterization used for that particual kind of denoising auto-encoder,
this would yield the constraint that $V^T = W$. That is, unless
we are using tied weights, we know that no such potential $E_0(x)$ exists,
and yet when running the experiments from section \ref{sec:simple-numerical-example}
we obtained much better results with untied weights.
To make things worse, it can also be demonstrated that the energy function that we
get from tied weights leads to a distribution that is not normalizable
(it has a divergent integral over $\mathbb{R}^d$).
In that sense, this suggests that we should not worry too much
about the exact parameterization of the denoising auto-encoder as long
as it has the required flexibility to approximate the optimal reconstruction
function sufficiently well.

\subsection{Relation to Denoising Score Matching}

There is a connection between our results and previous research
involving score matching for denoising auto-encoders.
We will summarize here the existing results from ~\cite{Vincent-NC-2011}
and show that, whereas they have shown that denoising auto-encoders
with a particular form estimated the score,
our results extend this to a very large family of estimators
(including the non-parametric case). This will provide
some reassurance given some of the potential issues
mentioned in section \ref{sec:limited-parameterization}.


Motivated by the analysis of denoising auto-encoders,
~\cite{Vincent-NC-2011} are concerned with the case where
we explicitly parametrize an energy function ${\cal E}(x)$,
yielding an associated score function $\psi(x)=-\frac{\partial {\cal E}(x)}{\partial x}$
and we stochastically corrupt the original samples $x \sim p$ to
obtain noisy samples $\tilde{x} \sim q_\sigma(\tilde{x}|x)$.
In particular, the article analyzes the case where $q_\sigma$ adds Gaussian noise of variance $\sigma^2$
to $x$. The main result is that minimizing the expected square
difference between $\psi(\tilde{x})$ and the score of $q_\sigma(\tilde{x}|x)$,
\[
  E_{x,\tilde{x}}[||\psi(\tilde{x}) - \frac{\partial \log q_\sigma(\tilde{x}|x)}{\partial \tilde{x}}||^2],
\]
is equivalent to performing {\em score matching}~\citep{Hyvarinen-2005}
with estimator $\psi(\tilde{x})$ and target density
$q_\sigma(\tilde{x})=\int q_\sigma(\tilde{x}|x) p(x) dx$, where $p(x)$
generates the training samples $x$. Note that when a finite training set is
used, $q_\sigma(\tilde{x})$ is simply a smooth of the empirical
distribution (e.g. the Parzen density with Gaussian kernel of width
$\sigma$).  When the corruption noise is Gaussian,
$\frac{q_\sigma(\tilde{x}|x)}{\partial
  \tilde{x}}=\frac{x-\tilde{x}}{\sigma^2}$, from which we can deduce that
if we define a reconstruction function
\begin{equation}
  r(\tilde{x})=\tilde{x}+\sigma^2 \psi(\tilde{x}),
\label{eq:r-psi}
\end{equation}
then the above expectation is
equivalent to
\[
  E_{x,\tilde{x}}[||\frac{r(\tilde{x})-\tilde{x}}{\sigma^2} - \frac{x-\tilde{x}}{\sigma^2}||^2]
  = \frac{1}{\sigma^2}
  E_{x,\tilde{x}}[||r(\tilde{x})-x||^2]
\]
which is the denoising criterion. This says that when the reconstruction
function $r$ is parametrized so as to correspond to the score $\psi$ of a
model density (as per eq.~\ref{eq:r-psi}, and where $\psi$ is a derivative
of some log-density), the denoising criterion on $r$
with Gaussian corruption noise is equivalent to 
score matching with respect to a smooth of the data generating density,
i.e., a regularized form of score matching. Note that this regularization
appears desirable, because matching the score of the empirical distribution
(or an insufficiently smoothed version of it) could yield undesirable results
when the training set is finite. Since score matching has been
shown to be a consistent induction principle~\citep{Hyvarinen-2005}, it means that this 
{\em denoising score matching}~\citep{Vincent-NC-2011,Kingma+LeCun-2010,Swersky-ICML2011} 
criterion recovers the underlying
density, up to the smoothing induced by the noise of variance $\sigma^2$.
By making $\sigma^2$ small, we can make the estimator arbitrarily good
(and we would expect to want to do that as the amount of training data
increases). Note the correspondance of this conclusion with the results
presented here, which show (1) the equivalence between the RCAE's regularization
coefficient and the DAE's noise variance $\sigma^2$, and (2) that minimizing the
equivalent analytic criterion (based on a contraction penalty) estimates
the score when ${\sigma^2}$ is small. The difference is that our result
holds even when $r$ is not parametrized as per eq.~\ref{eq:r-psi}, i.e.,
is not forced to correspond with the score function of a density.

\subsection{Estimating the Hessian}

Since we have $\frac{r(x)-x}{{\sigma^2}}$ as an estimator of the score,
we readily obtain that the Hessian of the log-density, can be
estimated by the Jacobian of the reconstruction function minus
the identity matrix:
\[
   \frac{\partial^2 \log p(x)}{\partial x^2} \approx (\frac{\partial r(x)}{\partial x} - I)/{\sigma^2}
\]
as shown by equation (\ref {eq:calcvarminloss-drx}) of Theorem \ref{thm:calcvarminloss}.

In spite of its simplicity, this result is interesting because it relates
the derivative of the reconstruction function, i.e., a Jacobian matrix,
with the second derivative of the log-density (or of the energy).  This
provides insights into the geometric interpretation of the reconstruction
function when the density is concentrated near a manifold.  In that case,
near the manifold the score is nearly 0 because we are near a ridge of
density, and the density's second derivative matrix tells us in which
directions the first density remains close to zero or increases.  The ridge
directions correspond to staying on the manifold and along these directions
we expect the second derivative to be close to 0.  In the orthogonal
directions, the log-density should decrease sharply while its first and
second derivatives would be large in magnitude and negative in directions
away from the manifold.

Returning to the above equation, keep in mind that in these derivations
$\sigma^2$ is near 0 and $r(x)$ is near $x$, so that $\frac{\partial
  r(x)}{\partial x}$ is close to the identity. In particular, in the ridge
(manifold) directions, we should expect $\frac{\partial r(x)}{\partial x}$
to be closer to the identity, which means that the reconstruction remains
faithful ($r(x)=x$) when we move on the manifold, and this corresponds to
the eigenvalues of $\frac{\partial r(x)}{\partial x}$ that are near 1,
making the corresponding eigenvalues of 
$\frac{\partial^2 \log p(x)}{\partial x^2}$ near 0.
On the other hand, in the directions orthogonal to the manifold, 
$\frac{\partial r(x)}{\partial x}$ should be smaller than 1,
making the corresponding eigenvalues of
$\frac{\partial^2 \log p(x)}{\partial x^2}$ negative.

Besides first and second derivatives of the density, other local 
properties of the density are its local mean and local covariance,
discussed in the Appendix, section~\ref{sec:local-moments}.

\section{Sampling with Metropolis-Hastings}
\label{seq:sampling}

\subsection{Estimating Energy Differences}
\label{subseq:energy-diff}
One of the immediate consequences of equation (\ref{eqn:score-estimator}) is that,
while we cannot easily recover the energy $E(x)$ itself,
it is possible to approximate the energy difference
\mbox{$E(x^*) - E(x)$} between two states $x$ and $x^*$.
This can be done by using a first-order Taylor approximation
\[
E(x^*) - E(x) = \frac{\partial E(x)}{\partial x}^T (x^* - x) + o(\left\| x^* - x \right\|).
\]
To get a more accurate approximation, we can also use
a path integral from $x$ to $x^*$ that we can discretize
in sufficiently many steps.
With a smooth path $\gamma(t):[0,1]\rightarrow \mathbb{R}^d$,
assuming that $\gamma$ stays in a region where our DAE/RCAE can
be used to approximate $\frac{\partial E}{\partial x}$ well enough,
we have that
\begin{equation}
\label{eq:energy-difference-by-path-integration}
E(x^*) - E(x) = \int_0^1 \left[ \frac{\partial E}{\partial x} \left( \gamma(t) \right)  \right]^T \gamma'(t) dt.
\end{equation}
The simplest way to discretize this path integral is to pick
points $\left\{x_i \right\}_{i=1}^{n}$ spread at even
distances on a straight line from $x_1 = x$ to $x_n = x^*$.
We approximate (\ref{eq:energy-difference-by-path-integration}) by
\begin{equation}
\label{eq:energy-difference-by-path-integration-discrete-version}
E(x^*) - E(x) \approx \frac{1}{n}\sum_{i=1}^n \left[ \frac{\partial E}{\partial x} \left( x_i \right)  \right]^T \left( x^* - x \right)
\end{equation}

\subsection{Sampling}
\label{subseq:sampling}

With equation (\ref{eq:energy-difference-by-path-integration})
from section \ref{subseq:energy-diff} we can perform approximate
sampling from the estimated distribution, using the score estimator
to approximate energy differences which are needed in the
Metropolis-Hastings accept/reject decision.
Using a symmetric proposal $q(x^*|x)$, the acceptance ratio is
\[
\alpha = \frac{p(x^*)}{p(x)} = \exp(-E(x^*)+E(x))
\]
which can be computed with (\ref{eq:energy-difference-by-path-integration})
or approximated with (\ref{eq:energy-difference-by-path-integration-discrete-version})
as long as we trust that our DAE/RCAE was trained properly and
has enough capacity to be a sufficiently good estimator of $\frac{\partial E}{\partial x}$.
An example of this process is shown in Figure \ref{fig:example-10d}
in which we sample from a density concentrated around
a 1-d manifold embedded in a space of dimension 10.
For this particular task, we have trained only DAEs and
we are leaving RCAEs out of this exercise.
Given that the data is roughly contained in the range $[-1.5,1.5]$
along all dimensions, we selected a training noise level
$\sigma_{\textrm{train}} = 0.1$ so that the noise would
have an appreciable effect while still being relatively small.
As required by Theorem \ref{thm:DAE-optimal-solution}, we have used
isotropic Gaussian noise of variance $\sigma_{\textrm{train}}^2$.

The Metropolis-Hastings proposal $q(x^*|x) = \mathcal{N}(0, \sigma_{\textrm{MH}}^2 I)$
has a noise parameter $\sigma_{\textrm{MH}}$ that needs
to be set. In the situation shown in Figure \ref{fig:example-10d},
we used $\sigma_{\textrm{MH}} = 0.1$.
After some hyperparameter tweaking and exploring
various scales for $\sigma_{\textrm{train}}, \sigma_{\textrm{MH}}$,
we found that setting both to be $0.1$ worked well.

When $\sigma_{\textrm{train}}$ is too large, the DAE trained learns
a ``blurry'' version of the density that fails
to represent the details that we are interested in.
The samples shown in Figure \ref{fig:example-10d} are very convincing
in terms of being drawn from a distribution that models well
the original density.
We have to keep in mind that Theorem \ref{thm:DAE-optimal-solution}
describes the behavior as
$\sigma_{\textrm{train}} \rightarrow 0$ so we would expect that the
estimator becomes worse when $\sigma_{\textrm{train}}$
is taking on larger values.
In this particular case with $\sigma_{\textrm{train}} = 0.1$,
it seems that we are instead modeling something like
the original density to which isotropic Gaussian noise of variance
$\sigma_{\textrm{train}}^2$ has been added.

In the other extreme, when $\sigma_{\textrm{train}}$ is too small, the DAE
is not exposed to any training example farther away from the density manifold.
This can lead to various kinds of strange behaviors when the sampling algorithm
falls into those regions and then has no idea what to do there and how to
get back to the high-density regions.
We come back to that topic in section \ref{subseq:spurious-maxima}.

It would certainly be possible to pick both a very small value for
$\sigma_{\textrm{train}} = \sigma_{\textrm{MH}} = 0.01$ to avoid the
spurius maxima problem illustrated in section \ref{subseq:spurious-maxima}.
However, this leads to the same kind of mixing problems that any
kind of MCMC algorithm has. Smaller values of $\sigma_{\textrm{MH}}$
lead to higher acceptance ratios but worse mixing properties.

\begin{figure}[h]
\centering
\makebox[0.20\textwidth][c]{\textsc{original}}
\makebox[0.20\textwidth][c]{\textsc{sampled}}
\makebox[0.20\textwidth][c]{\textsc{original}}
\makebox[0.20\textwidth][c]{\textsc{sampled}}

\includegraphics[width=0.20\textwidth]{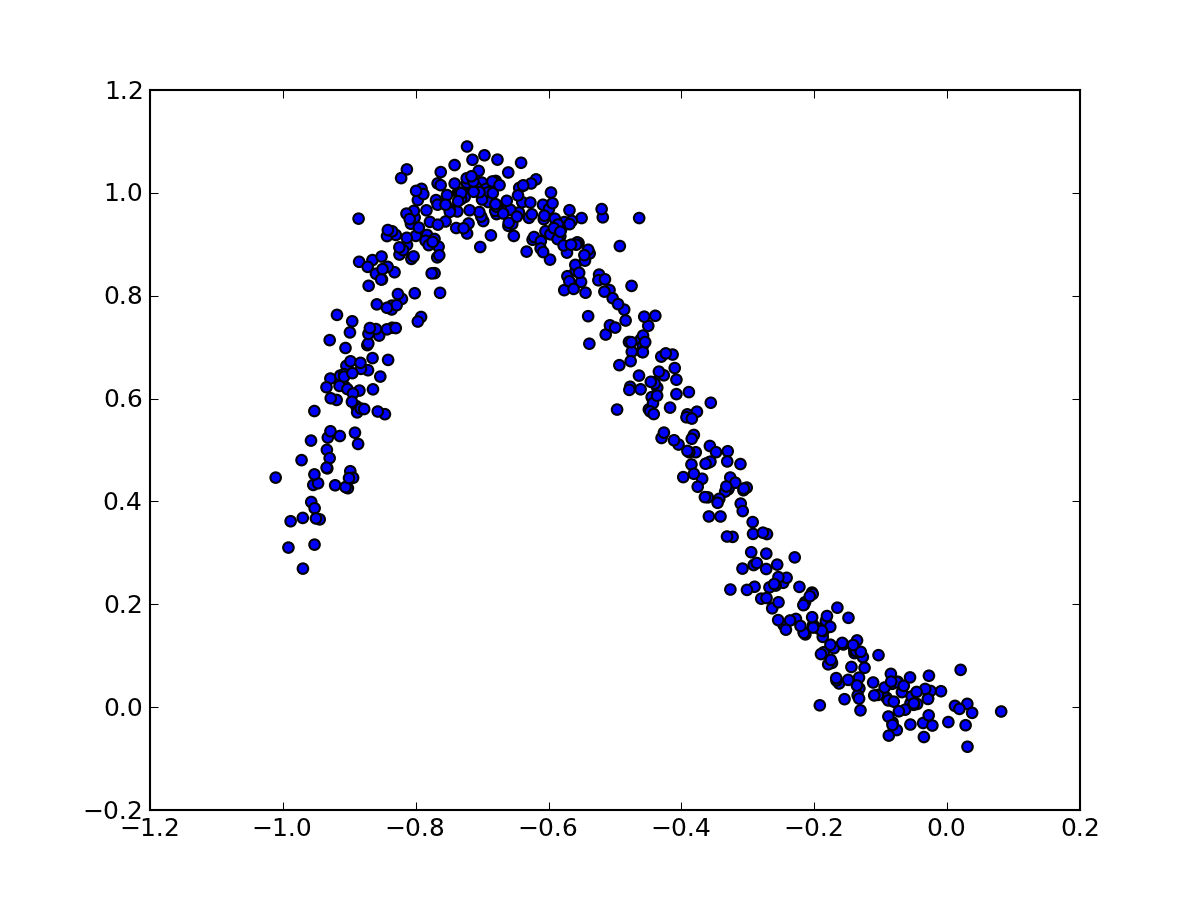}
\includegraphics[width=0.20\textwidth]{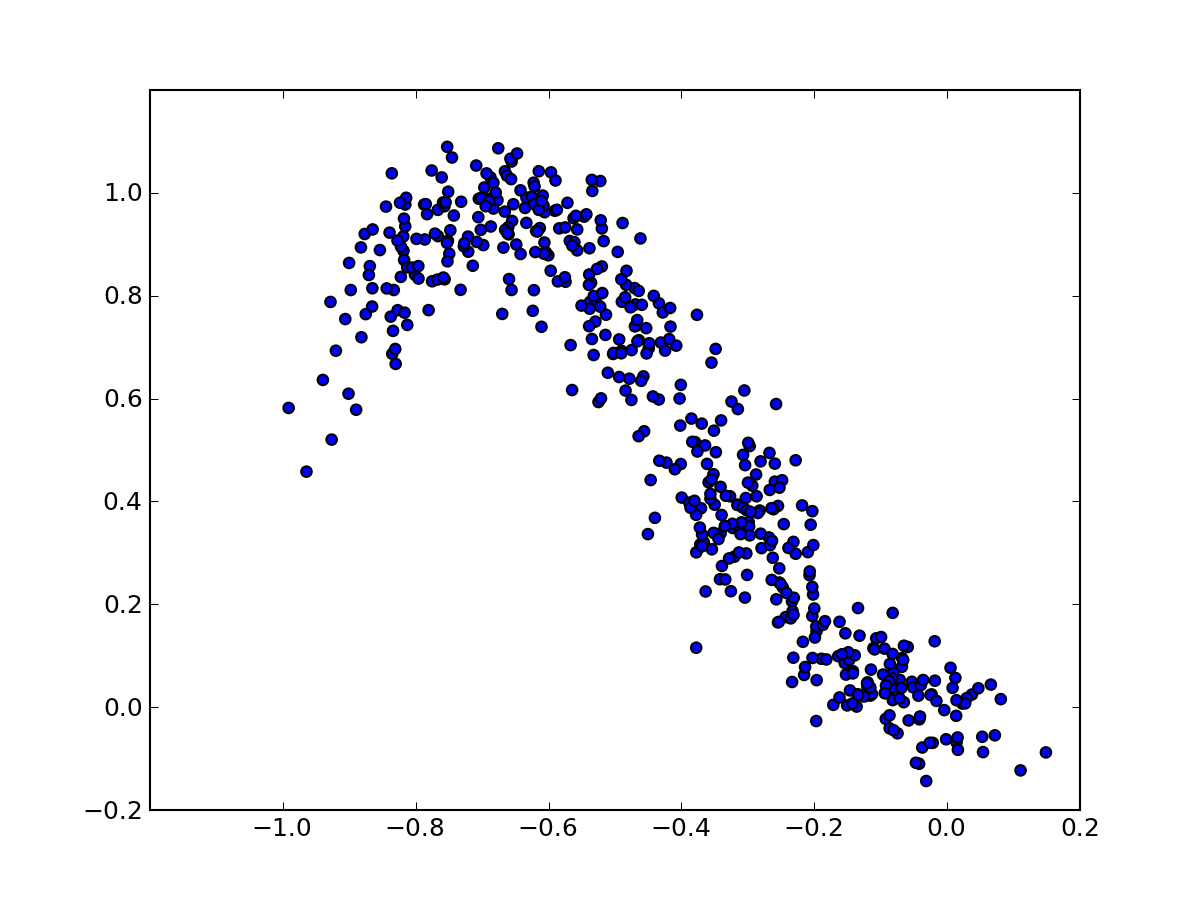}
\includegraphics[width=0.20\textwidth]{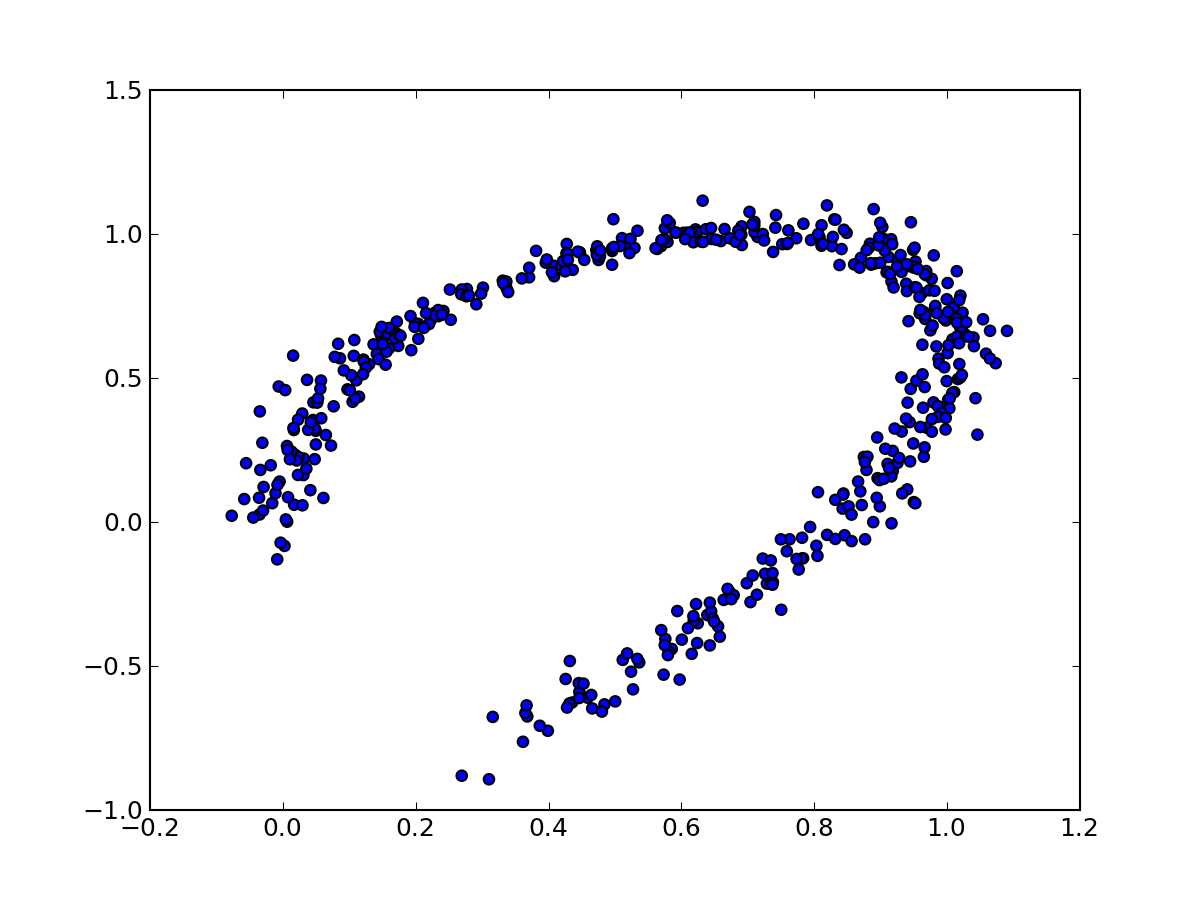}
\includegraphics[width=0.20\textwidth]{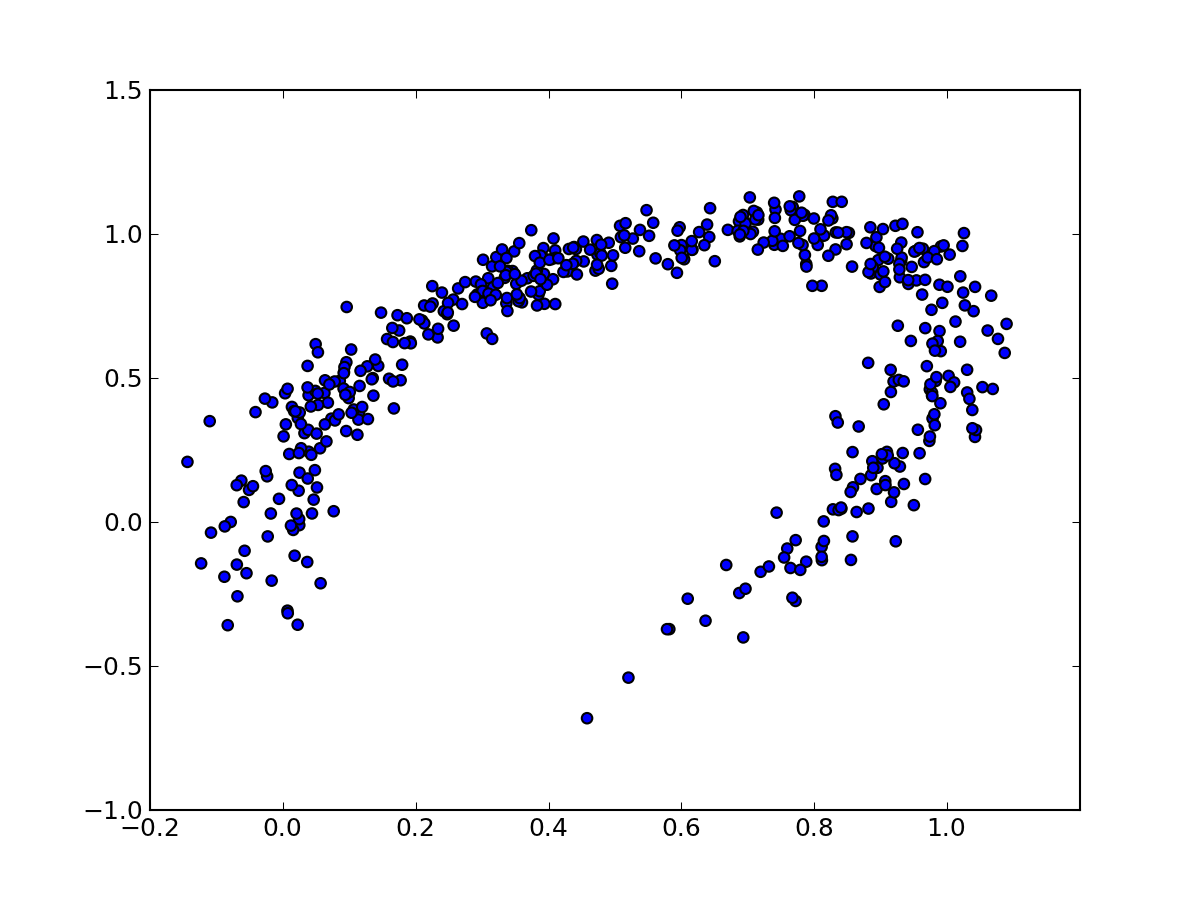}

\includegraphics[width=0.20\textwidth]{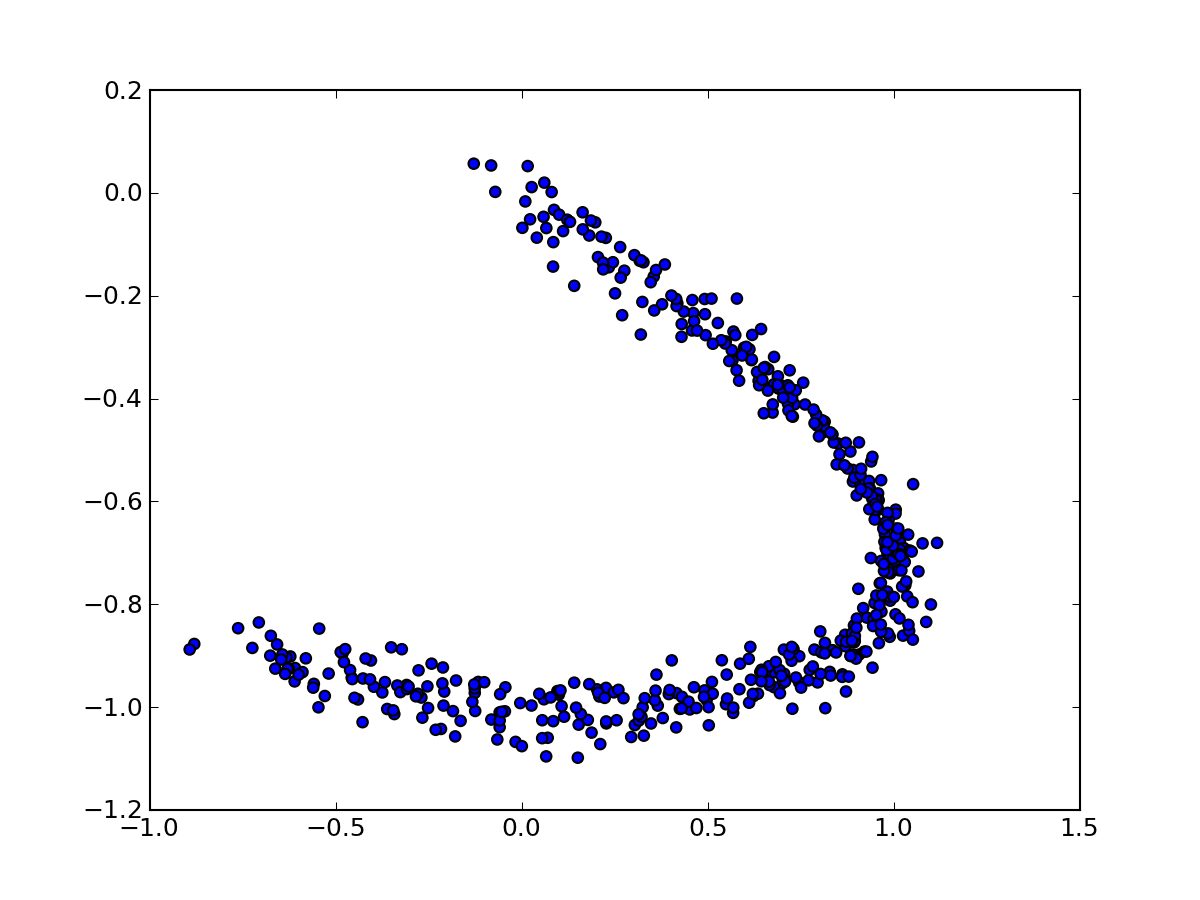}
\includegraphics[width=0.20\textwidth]{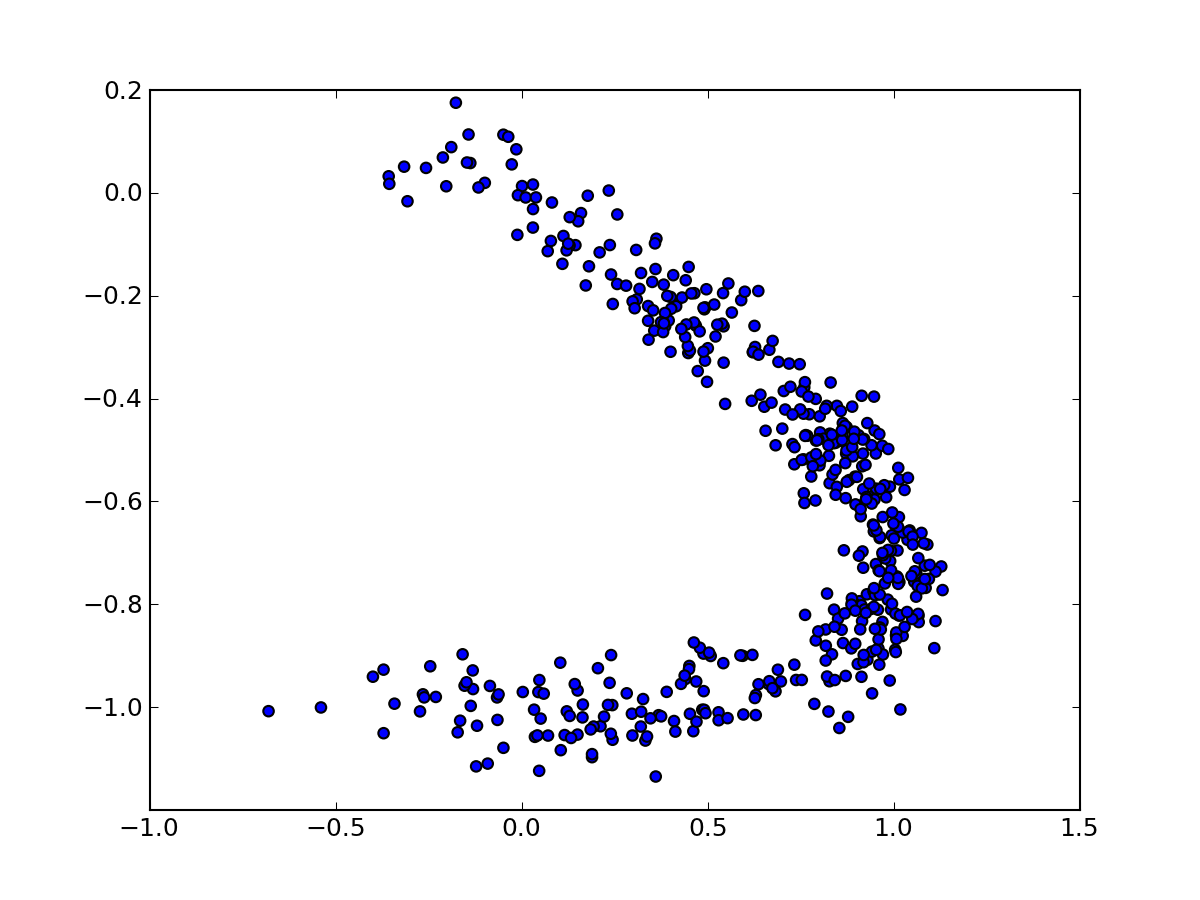}
\includegraphics[width=0.20\textwidth]{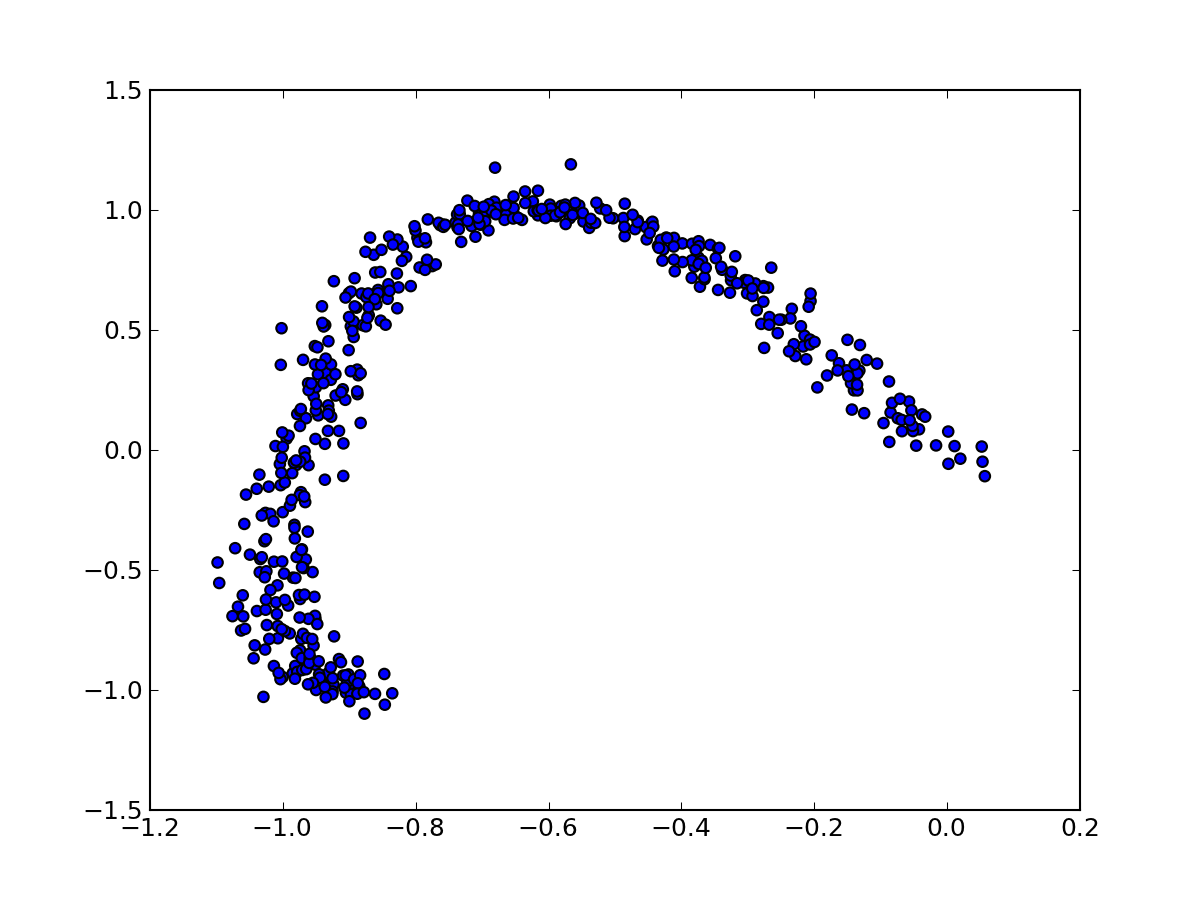}
\includegraphics[width=0.20\textwidth]{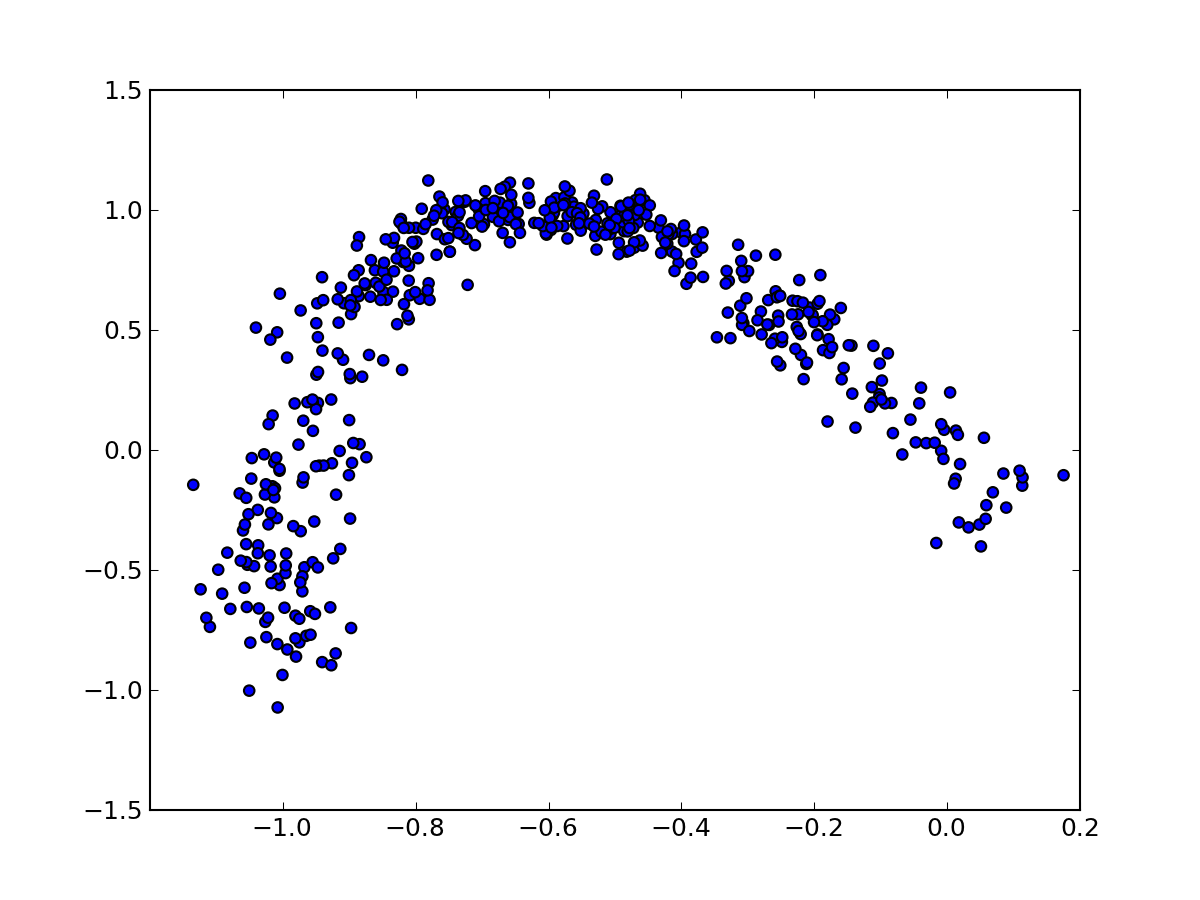}

\includegraphics[width=0.20\textwidth]{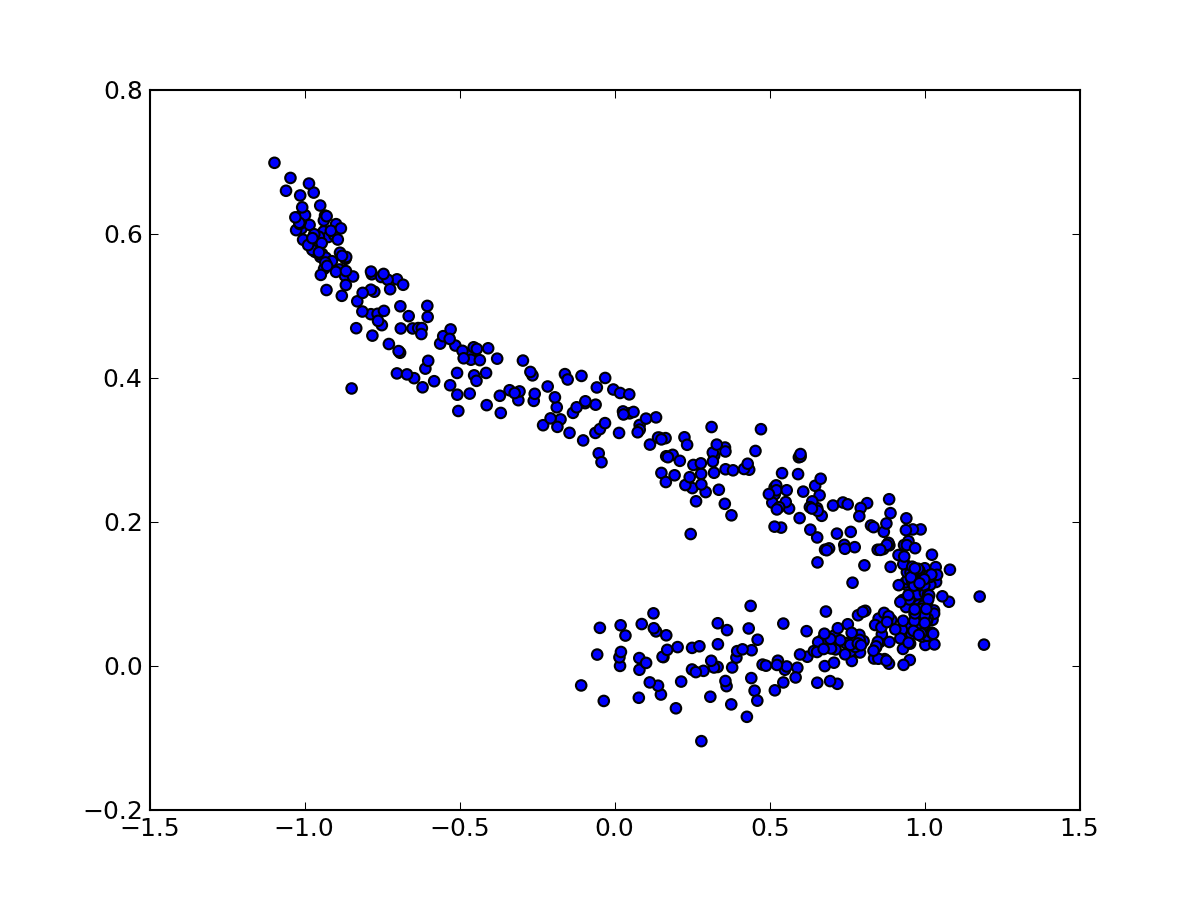}
\includegraphics[width=0.20\textwidth]{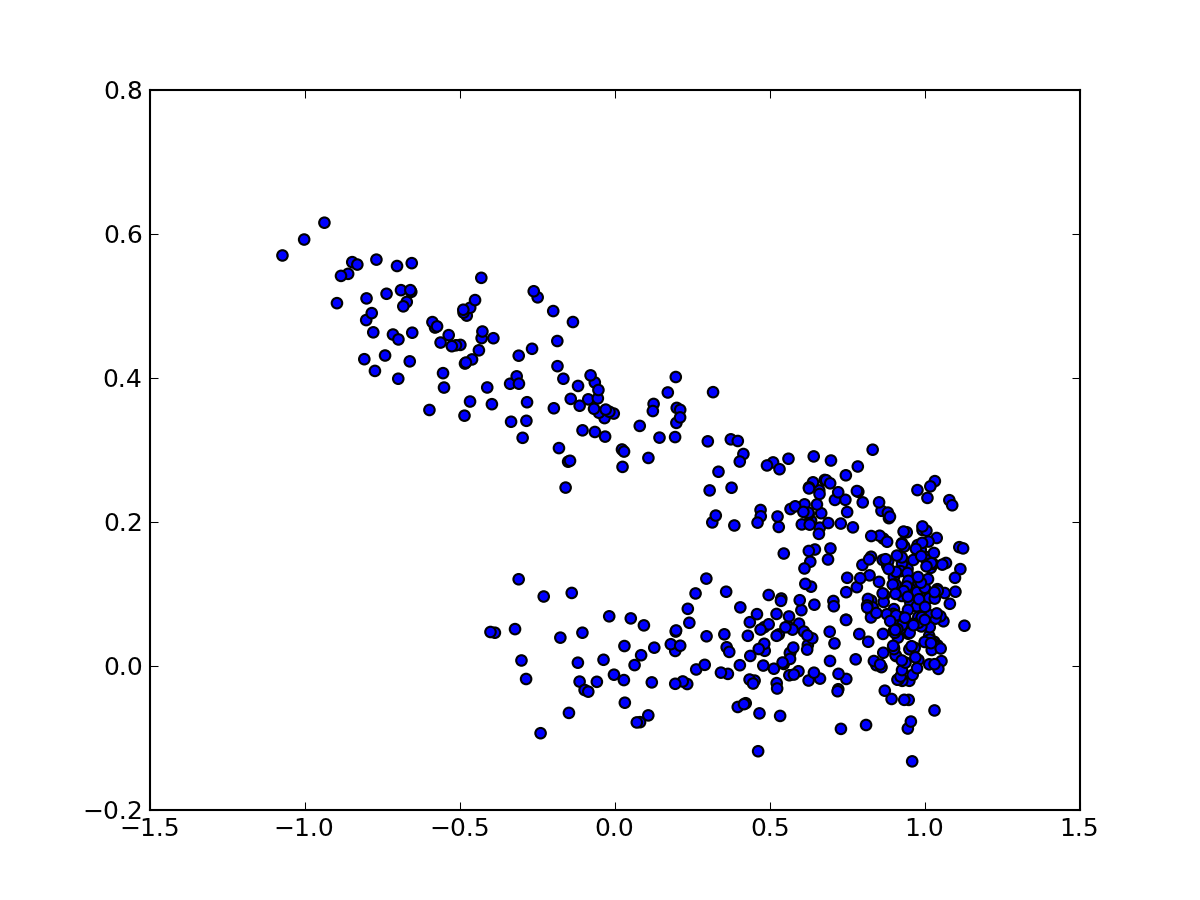}
\includegraphics[width=0.20\textwidth]{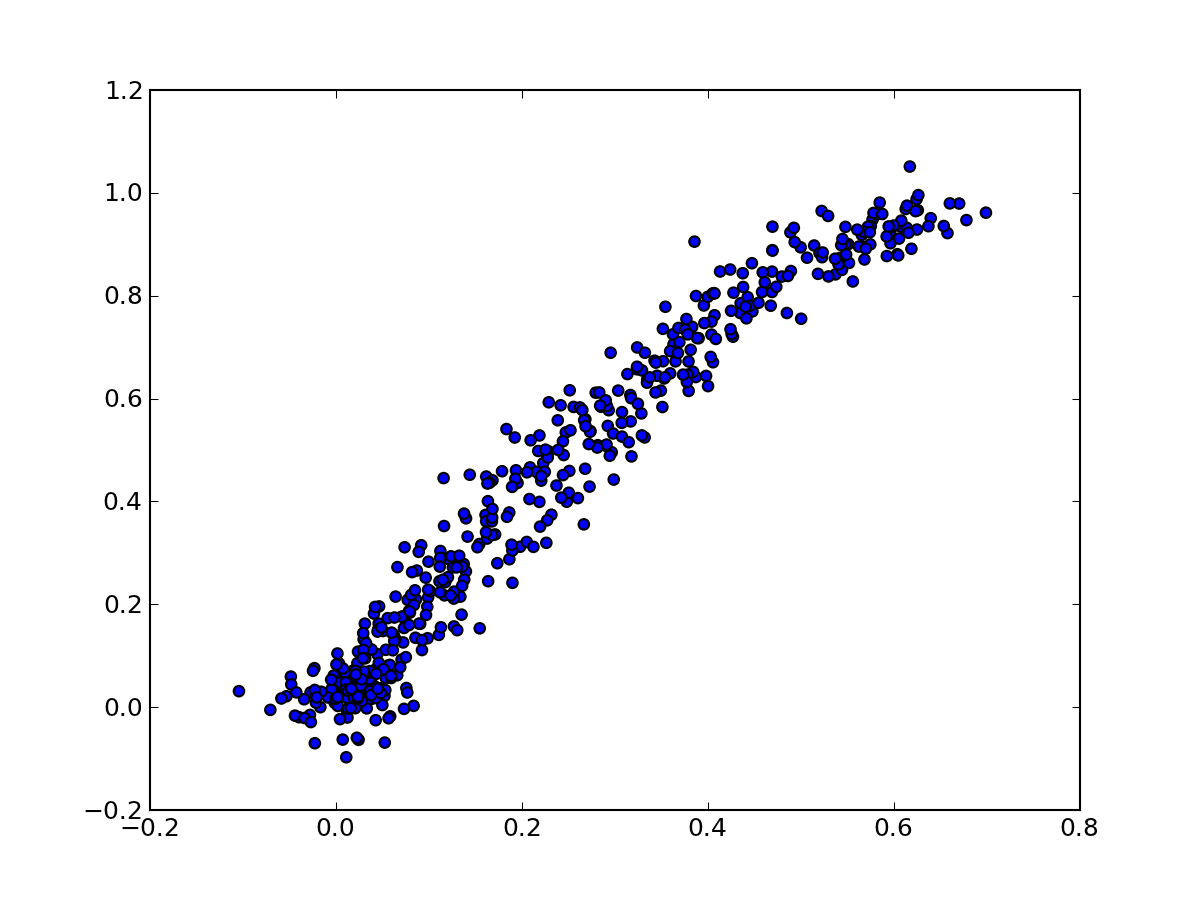}
\includegraphics[width=0.20\textwidth]{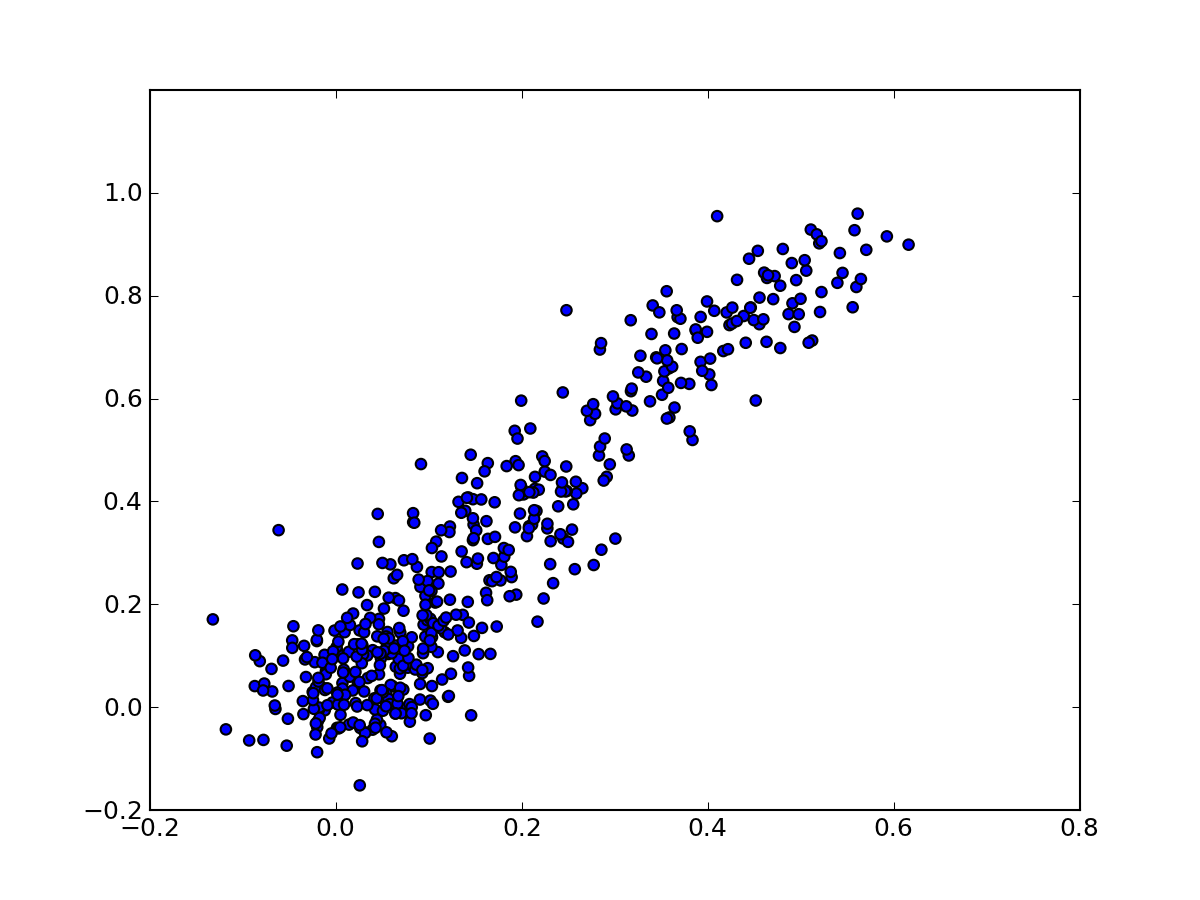}

\includegraphics[width=0.20\textwidth]{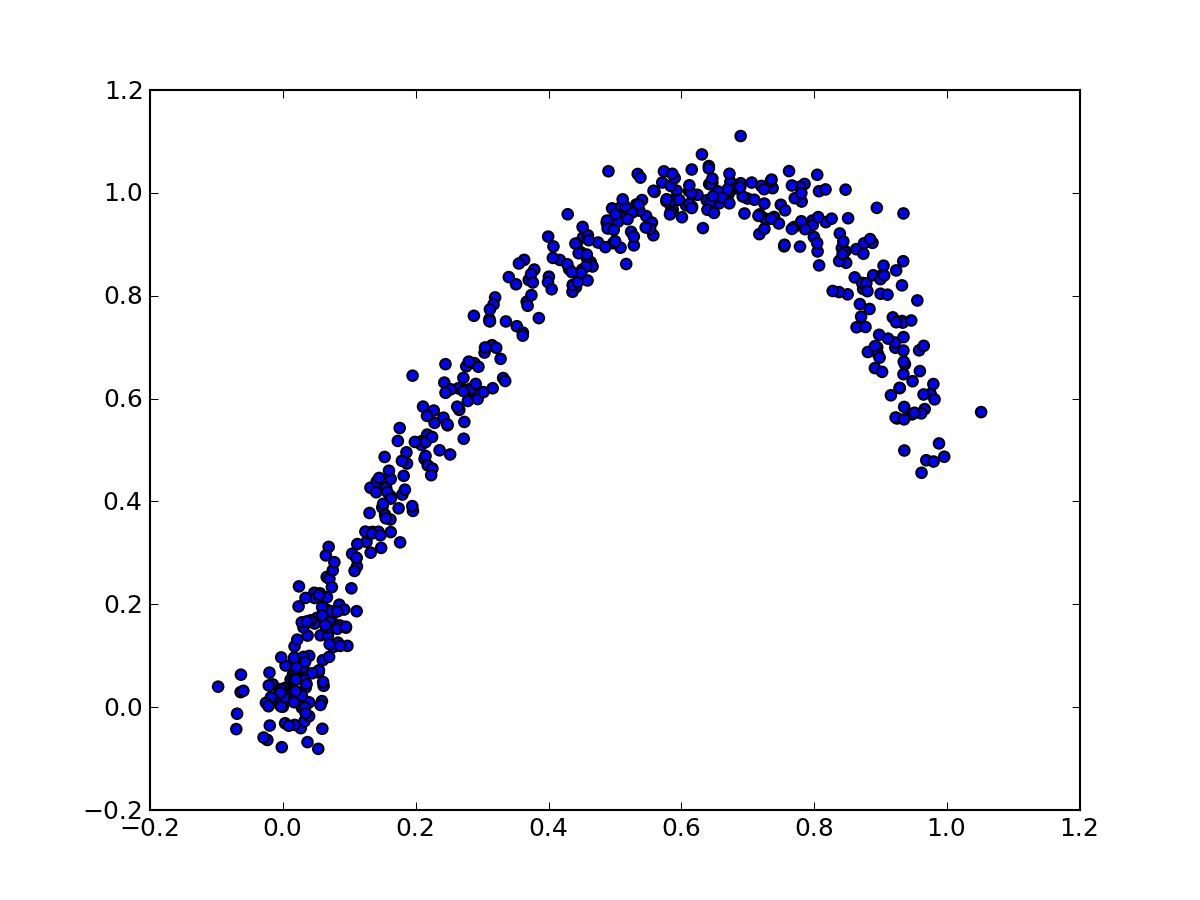}
\includegraphics[width=0.20\textwidth]{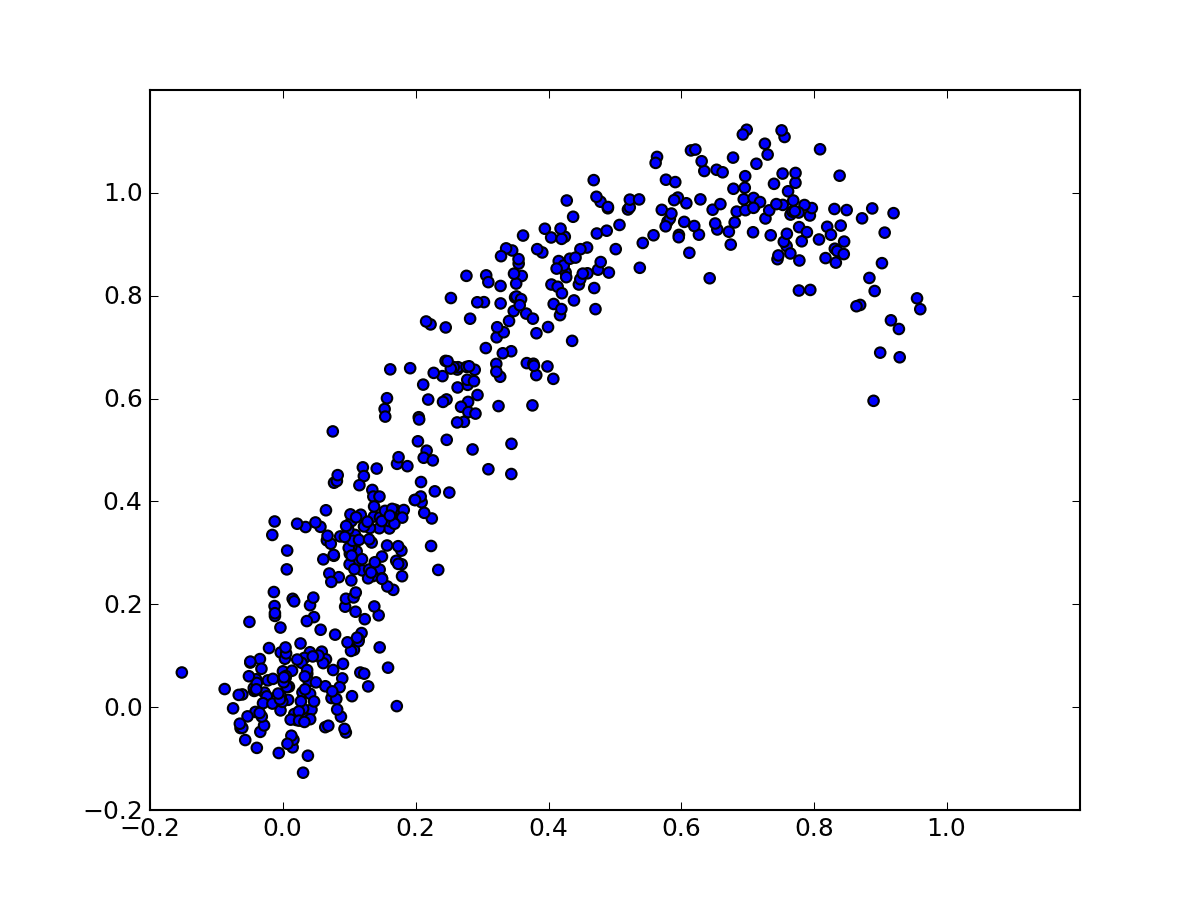}
\includegraphics[width=0.20\textwidth]{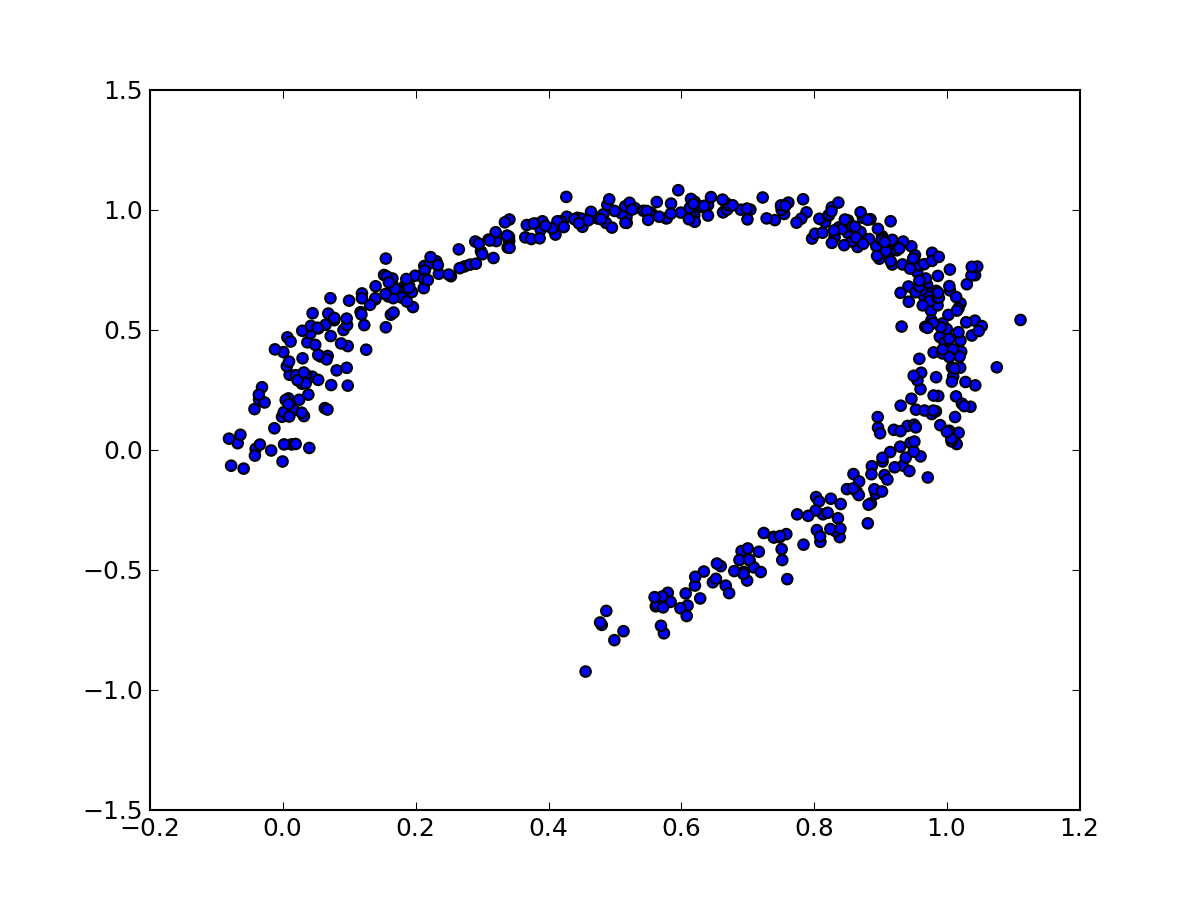}
\includegraphics[width=0.20\textwidth]{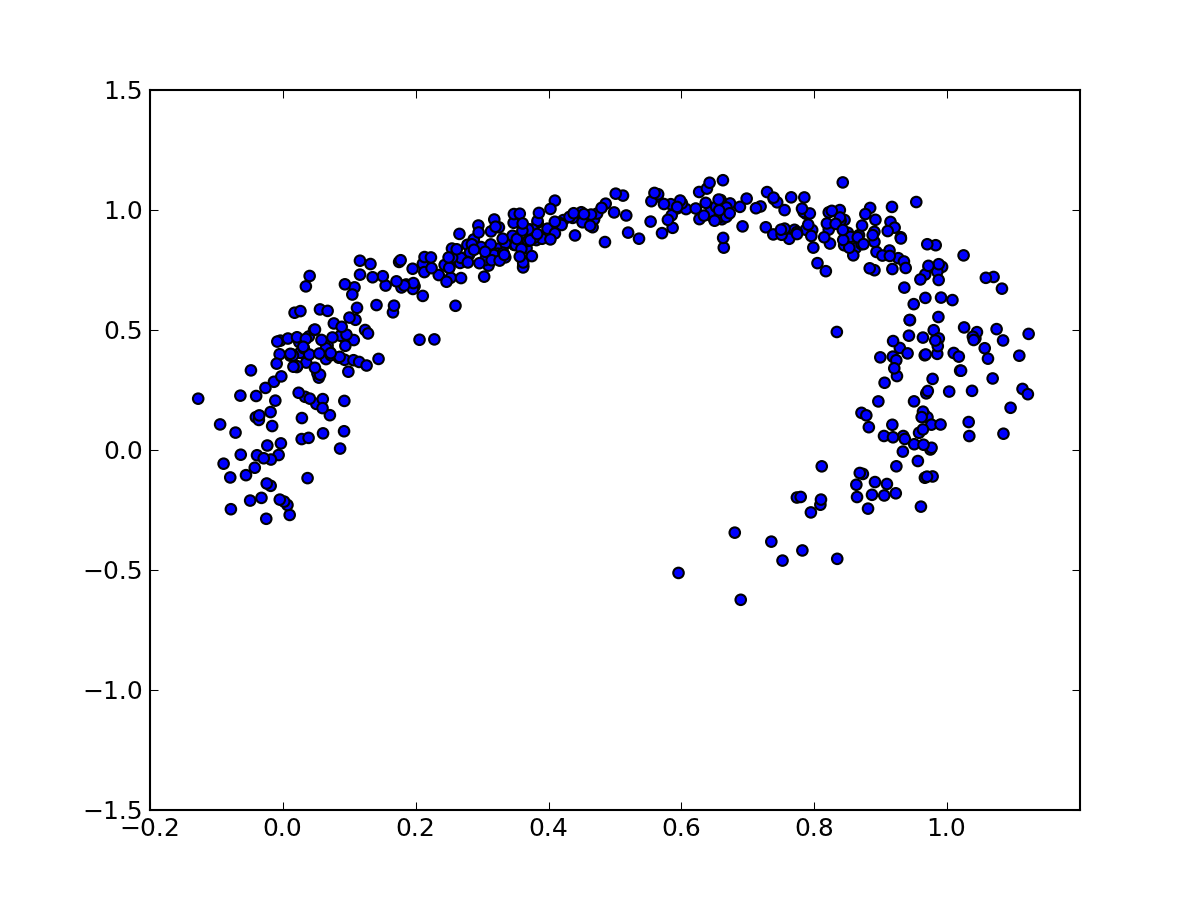}

\includegraphics[width=0.20\textwidth]{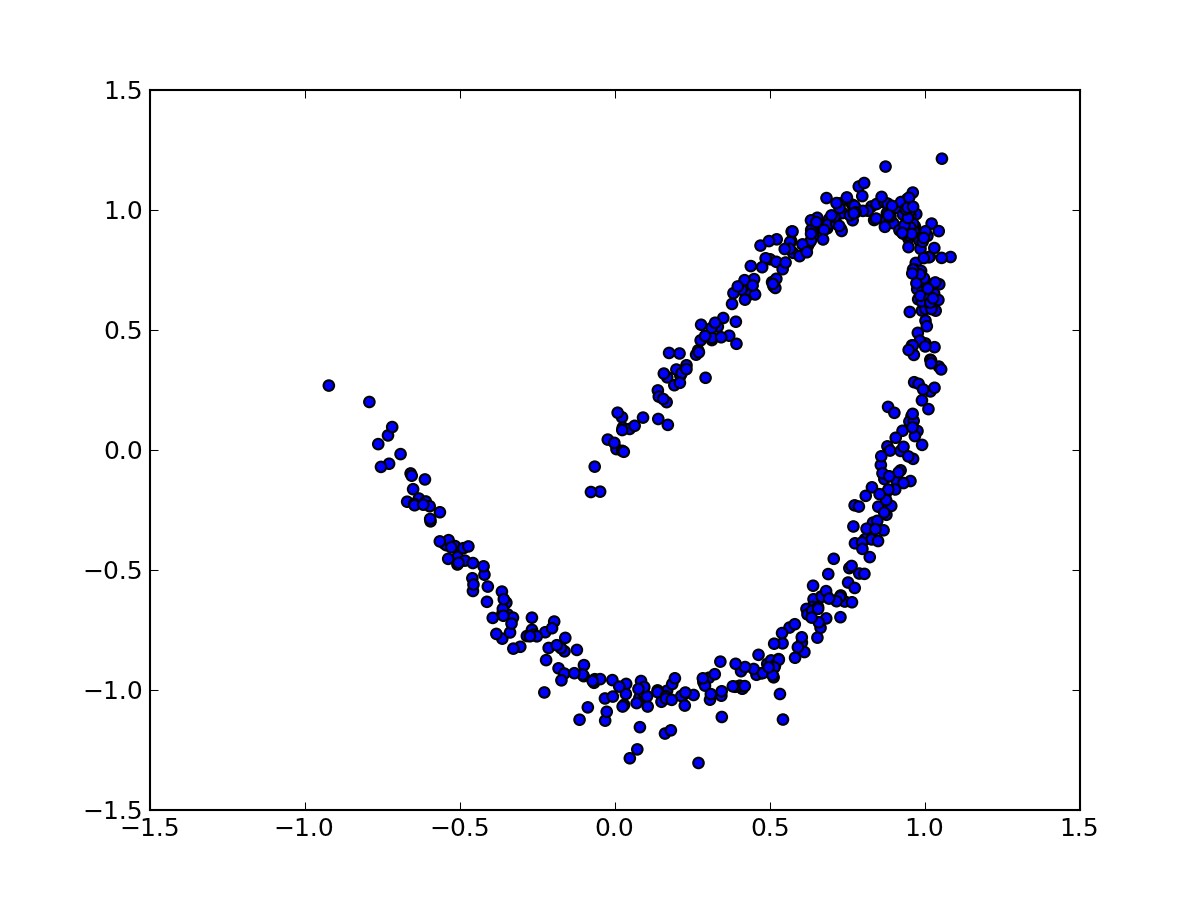}
\includegraphics[width=0.20\textwidth]{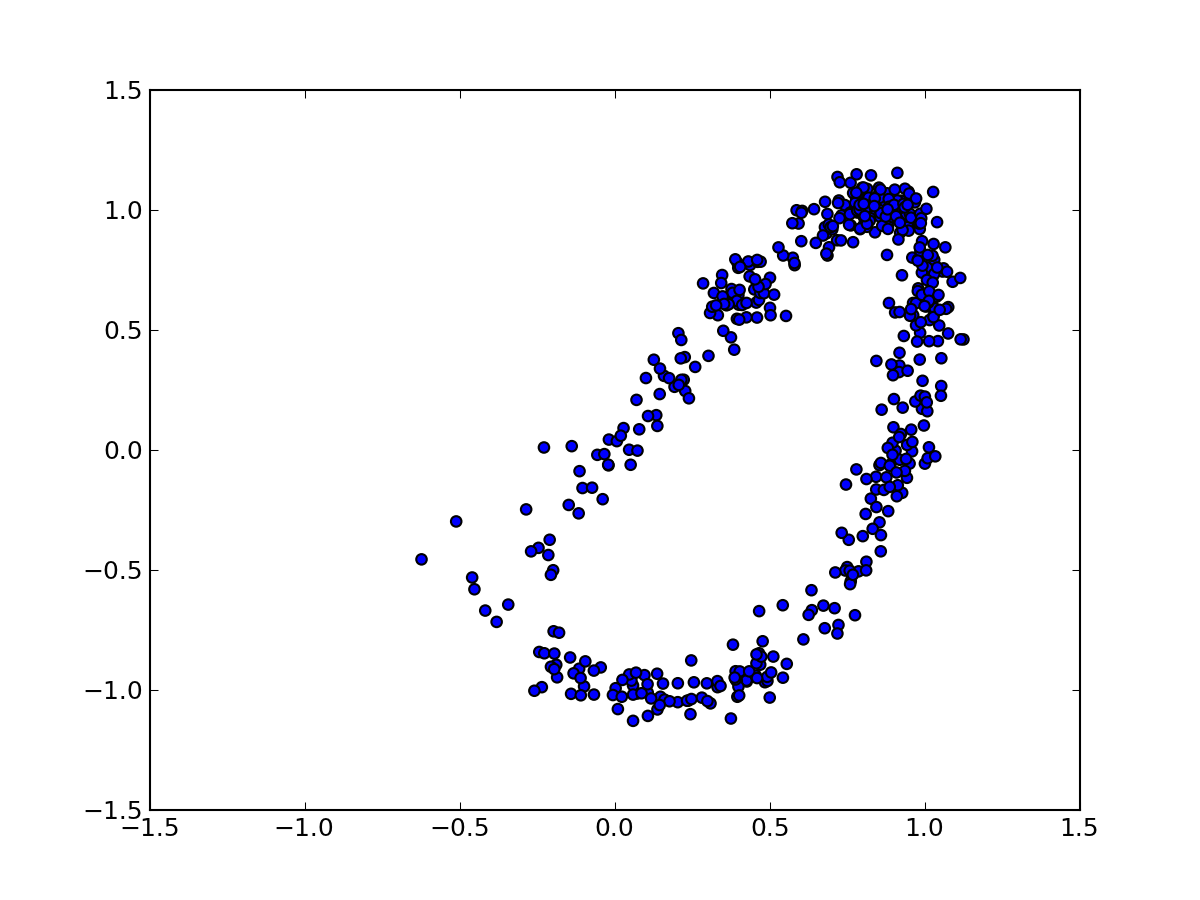}
\includegraphics[width=0.20\textwidth]{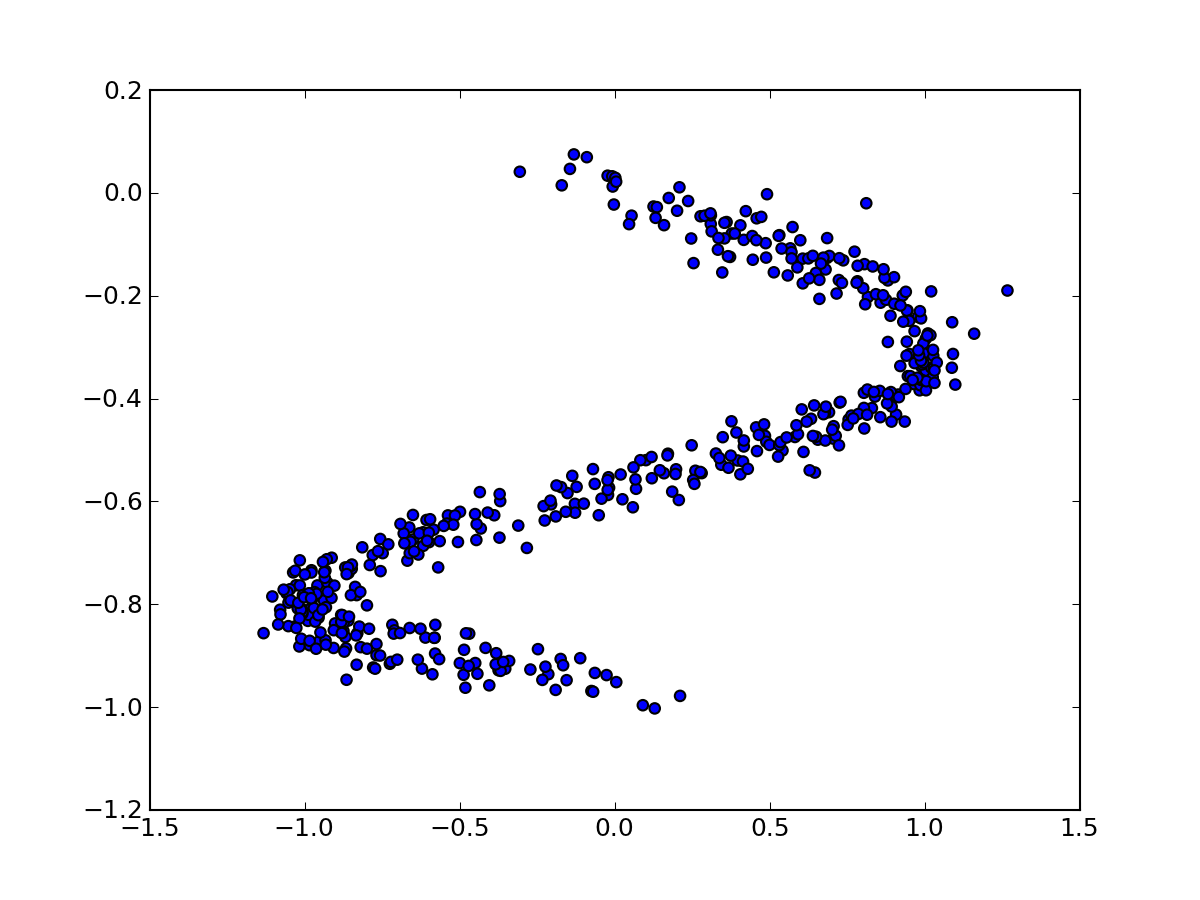}
\includegraphics[width=0.20\textwidth]{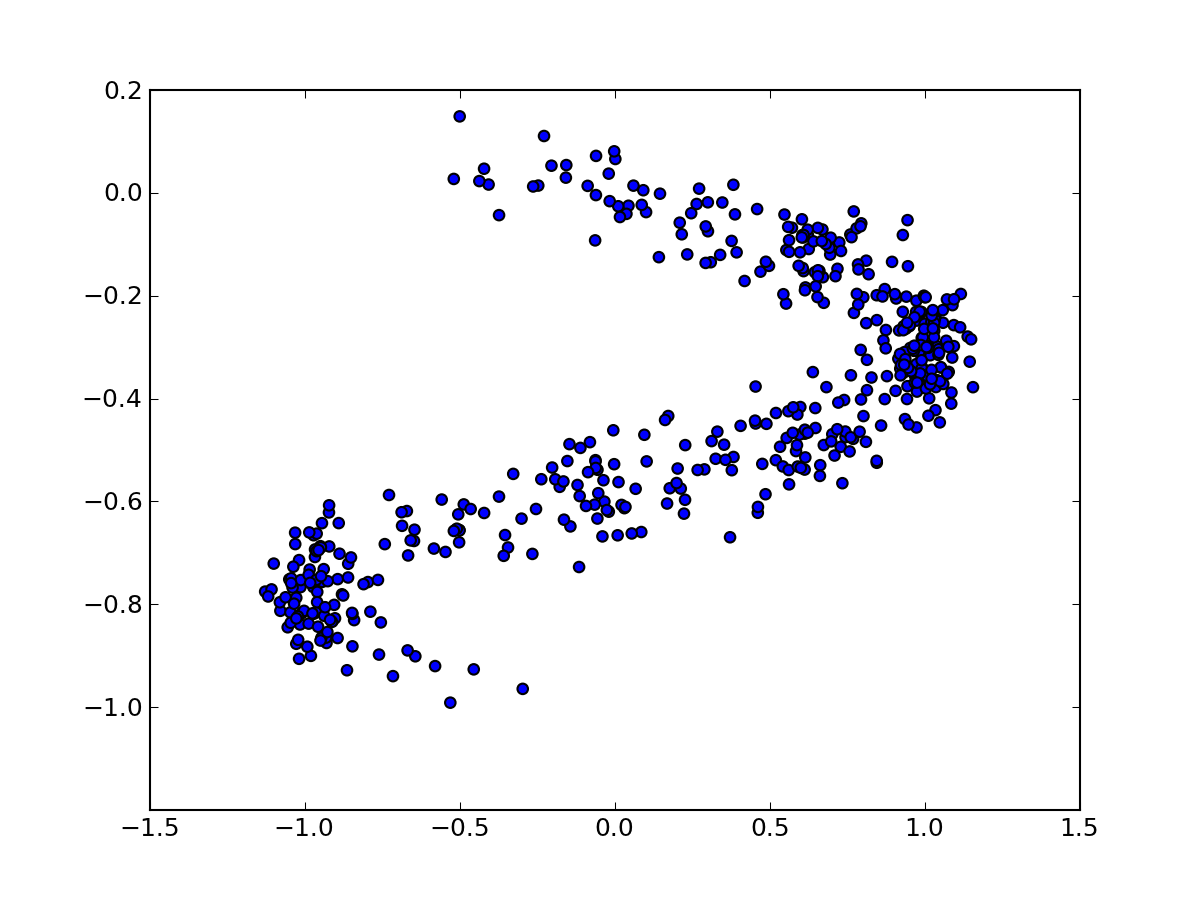}

\caption{
Samples drawn from the estimate of $\frac{\partial E}{\partial x}$ given by a DAE
by the Metropolis-Hastings method presented in section \ref{seq:sampling}.
By design, the data density distribution is concentrated along a 1-d manifold embedded in
a space of dimension 10. This data can be visualized in the plots above by plotting
pairs of dimensions $(x_0,x_1), \ldots, (x_8,x_9), (x_9,x_0)$, going 
in reading order from left to right and then line by line. For each pair of
dimensions, we show side by side the
original data (left) with the samples drawn (right).
}
\label{fig:example-10d}
\end{figure}

\subsection{Spurious Maxima}
\label{subseq:spurious-maxima}

There are two very real concerns with the sampling method discussed
in section \ref{subseq:sampling}.
The first problem is with the mixing
properties of MCMC and it is discussed in that section.
The second issue is with spurious probability maxima resulting from
inadequate training of the DAE. It happens when an auto-encoder
lacks the capacity to model the density with enough precision,
or when the training procedure ends up in a bad local minimum
(in terms of the DAE parameters).

This is illustrated in Figure \ref{fig:spurious-maxima}
where we show an example of a vector field $r(x)-x$ for
a DAE that failed to properly learn the desired behavior
in regions away from the spiral-shaped density.

\begin{figure}[h]
\centering
    \subfigure[][DAE misbehaving when away from manifold]{
        \includegraphics[width=0.45\textwidth]{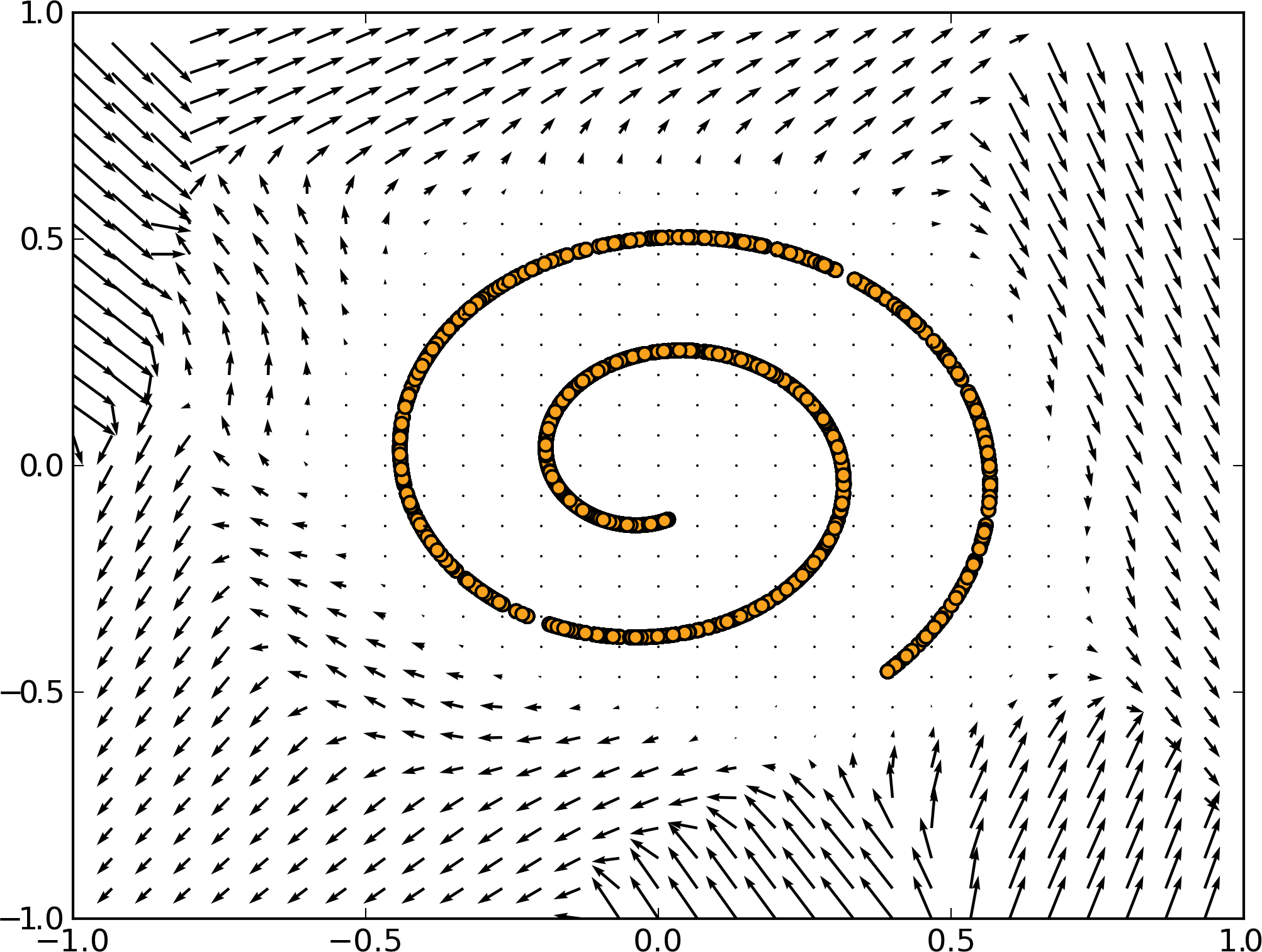}
        \label{fig1:dae-misbehaving}
    }
    \subfigure[][sampling getting trapped into bad attractor]{
        \includegraphics[width=0.45\textwidth]{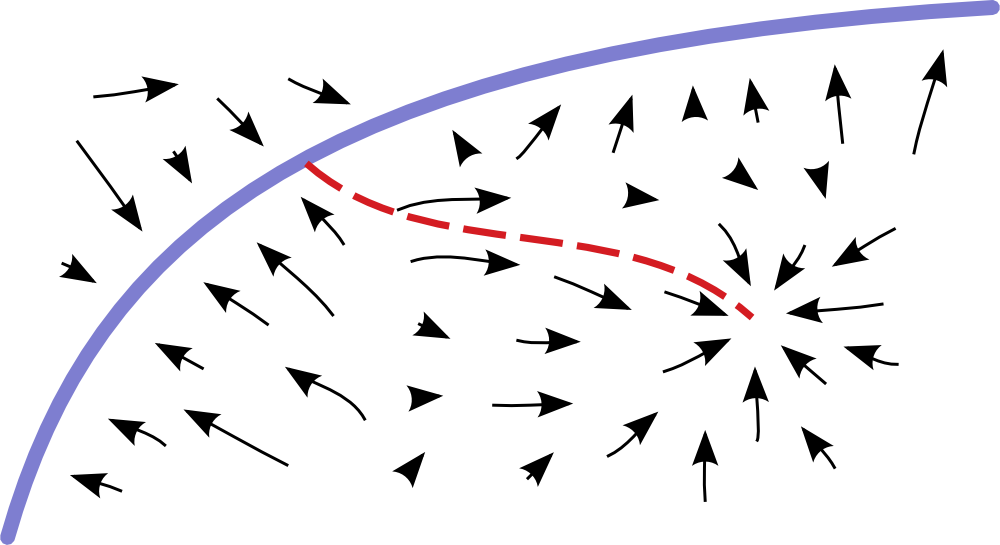}
        \label{fig1:walkback-trapped}
    }
\caption{
(a) On the left we show a $r(x)-x$ vector field similar to
that of the earlier Figure \ref{fig:two-spiral-graphics}.
The density is concentrated along a spiral manifold and we
should have the reconstruction function $r$ bringing
us back towards the density. In this case, it works well
in the region close to the spiral (the magnitude of the
vectors is so small that the arrows are shown as dots).
However, things are out of control in the regions outside.
This is because the level of noise used during training
was so small that not enough of the training examples were
found in those regions.
\newline
(b) On the right we sketch what may happen when
we follow a sampling procedure as described in section \ref{subseq:sampling}.
We start in a region of high density (in purple) and
we illustrate in red the trajectory that our samples may take.
In that situation, the DAE/RCAE was not trained properly.
The resulting vector field does not reflect the density accurately
because it should not have this attractor (i.e. stable fixed point)
outside of the manifold on which the density is concentrated.
Conceptually, the sampling procedure visits that spurious attractor because
it assumes that it corresponds to a region of high probability.
In some cases, this effect is regrettable but not catastrophic,
but in other situations we may end up with completely unusable samples.
In the experiments, training with enough of the examples involving 
sufficiently large corruption noise typically eliminates that problem.
}
\label{fig:spurious-maxima}
\end{figure}

\section{Conclusion}

Whereas auto-encoders have long been suspected of capturing information
about the data generating density, this work has clarified what some
of them are actually doing, showing that they can actually implicitly
recover the data generating density altogether.
We have shown that regularized auto-encoders such as the denoising
auto-encoder and a form of contractive auto-encoder are closely related to
each other and estimate local properties of the data generating density:
the first derivative (score) and second derivative of the log-density, as
well as the local mean. This contradicts the previous interpretation of
reconstruction error as being an energy function~\citep{ranzato-08} but
is consistent with our experimental findings.
Our results do not require the
reconstruction function to correspond to the derivative of an energy
function as in~\citet{Vincent-NC-2011}, but hold simply by virtue of
minimizing the regularized reconstruction error training criterion. This
suggests that minimizing a regularized reconstruction error may be an
alternative to maximum likelihood for unsupervised learning, avoiding the
need for MCMC in the inner loop of training, as in RBMs and deep Boltzmann
machines, analogously to score matching~\citep{Hyvarinen-2005,Vincent-NC-2011}. 
Toy experiments have confirmed that a good estimator of the
density can be obtained when this criterion is non-parametrically
minimized. The experiments have also confirmed that an MCMC could be
setup that approximately samples from the estimated model, by 
estimating energy differences to first order (which only requires the score)
to perform approximate Metropolis-Hastings MCMC.

Many questions remain open and deserve futher study. A big question is how
to generalize these ideas to discrete data, since we have heavily relied
on the notions of scores, i.e., of derivatives with respect to $x$.
A natural extension of the notion of score that could be applied to
discrete data is the notion of {\em relative energy}, or energy difference
between a point $x$ and a perturbation $\tilde{x}$ of $x$. This notion
has already been successfully applied to obtain the equivalent of score matching
for discrete models, namely ratio matching~\citep{Hyvarinen-2007}.
More generally, we would like to generalize to any form of reconstruction
error (for example many implementations of auto-encoders use a Bernouilli
cross-entropy as reconstruction loss function) and any (reasonable) form
of corruption noise (many implementations use masking or salt-and-pepper
noise, not just Gaussian noise). More fundamentally, the need to rely
on $\sigma \rightarrow 0$ is troubling, and getting rid of this limitation
would also be very useful. A possible solution to this limitation,
as well as adding the ability to handle both discrete and continuous
variables, has recently been proposed while this
article was under review~\citep{Bengio-et-al-arxiv-2013}.

It would also be interesting to generalize the results presented here to
other regularized auto-encoders besides the denoising and contractive types.
In particular, the commonly used sparse auto-encoders seem to fit the qualitative
pattern illustrated in Figure \ref{fig:1D-autoencoder} where a score-like vector
field arises out of the opposing forces of minimizing reconstruction error
and regularizing the auto-encoder.

We have mostly considered the harder case where the auto-encoder
parametrization does not guarantee the existence of an analytic formulation
of an energy function. It would be interesting to compare experimentally
and study mathematically these two formulations to assess how much is lost
(because the score function may be somehow inconsistent) or gained (because
of the less constrained parametrization).

\subsubsection*{Acknowledgements}
The authors thank Salah Rifai, 
Max Welling, Yutian Chen  
and Pascal Vincent for 
fruitful discussions, and acknowledge the funding support from NSERC,
Canada Research Chairs and CIFAR.

{\small
\bibliography{strings,strings-shorter,ml,aigaion-shorter}
}


\section{Appendix}

\subsection{Optimal DAE solution}

\setcounter{theorem}{0}
\begin{theorem}
\label{thm_app:DAE-optimal-solution}
Let $p$ be the probability density function of the data.
If we train a $DAE$ using the expected quadratic loss
and corruption noise $N(x) = x+\epsilon$ with
\[
\epsilon \sim \mathcal{N}\left(0, \sigma^2 I \right),
\]
then the optimal reconstruction function $r^{*}(x)$ will be given by
\begin{equation}
r^*(x) = \frac{ \mathbb{E}_{\epsilon} \left[ p(x-\epsilon) (x - \epsilon) \right]}{\mathbb{E}_{\epsilon} \left[ p(x-\epsilon) \right]} \label{eqn_app:opt-ratio}
\end{equation}
for values of $x$ where $p(x) \neq 0$.

Moreover, if we consider how the optimal reconstruction function
$r^{*}_{\sigma}(x)$ behaves asymptotically as $\sigma \rightarrow 0$,
we get that
\begin{equation}
r^*_\sigma(x) = x + \sigma^2 \frac{\partial \log p(x)}{\partial x} + o(\sigma^2) \hspace{1em} \textrm{as} \hspace{1em} \sigma \rightarrow 0. \nonumber 
\end{equation}
\end{theorem}

\begin{proof}
The first part of this proof is to get to equation (\ref{eqn_app:opt-ratio})
without assuming that $\sigma \rightarrow 0$.

By using an auxiliary variable $\tilde{x} = x + \epsilon$,
we can rewrite this loss in a way that puts
the quantity $r(\tilde{x})$ in focus and allows us to
perform the minimization with respect to each choice of
$r(\tilde{x})$ independantly.
That is, we have that
\begin{equation*}
\mathcal{L}_{\textrm{DAE}} = \int_{\mathbb{R}^d} \mathbb{E}_{\epsilon \sim \mathcal{N}\left(0, \sigma^2 I \right)} \left[ p(\tilde{x}-\epsilon) \left\|r(\tilde{x}) - \tilde{x} + \epsilon \right\|^2_2 \right] d\tilde{x}
\end{equation*}
which can be differentiated with respect to the
quantity $r(\tilde{x})$ and set to be equal to $0$.
Denoting the optimum by $r^*(\tilde{x})$, we get

\begin{equation*}
0 = \mathbb{E}_{\epsilon \sim \mathcal{N}\left(0, \sigma^2 I \right)} \left[ p(\tilde{x}-\epsilon) \left(r^*(\tilde{x}) - \tilde{x} + \epsilon \right) \right]
\end{equation*}
\begin{equation*}
\mathbb{E}_{\epsilon \sim \mathcal{N}\left(0, \sigma^2 I \right)} \left[ p(\tilde{x}-\epsilon) r^*(\tilde{x}) \right] = \mathbb{E}_{\epsilon \sim \mathcal{N}\left(0, \sigma^2 I \right)} \left[ p(\tilde{x}-\epsilon) (\tilde{x} - \epsilon) \right]
\end{equation*}
\begin{eqnarray}
r^*(\tilde{x}) & = & \frac{\mathbb{E}_{\epsilon \sim \mathcal{N}\left(0, \sigma^2 I \right)} \left[ p(\tilde{x}-\epsilon) (\tilde{x} - \epsilon) \right]}{\mathbb{E}_{\epsilon \sim \mathcal{N}\left(0, \sigma^2 I \right)} \left[ p(\tilde{x}-\epsilon) \right]}. \label{eqn_app:alternative-dae-loss-4}
\end{eqnarray}
We used $\tilde{x}$ out of convenience in equation (\ref{eqn_app:alternative-dae-loss-4}).
Theorem \ref{thm_app:DAE-optimal-solution} is just nicer to read when using $r(x)$, and we switch back to using $r(x)$ for the
rest of this proof.

Now, for the second part of this proof, we study the behavior of the solution $r_\sigma^*(x)$ as $\sigma$ approaches $0$.
We start with equation (\ref{eqn_app:alternative-dae-loss-4}) that we rewrite in a way to be able to pull out the leading
term $x$. Given the symmetry of the distribution of $\epsilon \sim \mathcal{N}\left(0, \sigma^2 I \right)$,
we can also simplify equation (\ref{eqn_app:alternative-dae-loss-4}) by converting all the $-\epsilon$ into $+\epsilon$.

\begin{eqnarray*}
\frac{\mathbb{E}_\epsilon \left[ p(x+\epsilon) (x + \epsilon) \right]}{\mathbb{E}_\epsilon \left[ p(x+\epsilon) \right]} & = & \frac{\mathbb{E}_\epsilon \left[ p(x+\epsilon) x \right]}{\mathbb{E}_\epsilon \left[ p(x+\epsilon) \right]} + \frac{\mathbb{E}_\epsilon \left[ p(x+\epsilon) \epsilon \right]}{\mathbb{E}_\epsilon \left[ p(x+\epsilon) \right]} \\
& = & x + \frac{\mathbb{E}_\epsilon \left[ p(x+\epsilon) \epsilon \right]}{\mathbb{E}_\epsilon \left[ p(x+\epsilon) \right]}
\end{eqnarray*}

We now look at the Taylor expansion of $p(x+\epsilon)$ around $x$.
We perform this expansion inside of the expectation, where $\epsilon$
just represents as small quantity of scale $\sigma$.
\begin{equation}
p(x+\epsilon) = p(x) + \nabla p(x)^T \epsilon + \frac{1}{2}\epsilon^T \frac{\partial^2 p(x)}{\partial x^2} \epsilon + o(\epsilon^2) \hspace{1em} \textrm{as} \hspace{1em} \epsilon \rightarrow 0
\end{equation}

When taking the expectation of $p(x+\epsilon)$ with respect to $\epsilon$,
we get a zero for all the terms containing an odd power of $\epsilon$,
for reasons of symmetry. When we take the expectation of $p(x+\epsilon)\epsilon$
instead, all the terms with an even power of $\epsilon$ in the above expectation
will vanish. Thus, we get that
\begin{eqnarray*}
\mathbb{E}_\epsilon \left[ p(x+\epsilon) \epsilon \right] & = & \mathbb{E}_\epsilon[\epsilon \epsilon^T] \nabla p(x)  + \mathbb{E}_\epsilon[o(\epsilon^2)] = \sigma^2 \nabla p(x) + o(\sigma^2) \\
\mathbb{E}_\epsilon \left[ p(x+\epsilon) \right] & = & p(x) + \mathcal{O}(\sigma^2) \\
\end{eqnarray*}

Provided that $p(x) \neq 0$, our quotient can be rewritten as
\begin{equation}
\label{eqn_app:simplified_quotient}
\frac{\sigma^2 \nabla p(x) + o(\sigma^2)}{p(x) + \mathcal{O}(\sigma^2)} = \frac{\sigma^2 \nabla \log p(x) + o(\sigma^2)}{1 + \mathcal{O}(\sigma^2)}.
\end{equation}

We can use the basic geometric series expansion to write
\begin{equation*}
\frac{1}{1 + \mathcal{O}(\sigma^2)} = 1 - \mathcal{O}(\sigma^2) = 1 + \mathcal{O}(\sigma^2)
\end{equation*}
so that equation (\ref{eqn_app:simplified_quotient}) is now
\begin{equation*}
\frac{\sigma^2 \nabla \log p(x) + o(\sigma^2)}{1 + \mathcal{O}(\sigma^2)} = \sigma^2 \nabla \log p(x) + o(\sigma^2).
\end{equation*}
This is the result that we wanted.
\end{proof}

Note that $p(x)$ was treated as a constant when we studied the asymptotic behavior
as $\sigma \rightarrow 0$. In the Taylor expansion around some given $x$,
we want to stuff all the higher-order derivatives of $p(x)$ into the asymptotic
remainder term. It is quite possible that the size of the $\sigma$ required to
do so would depend on the particular $x$, that there would not be a uniform $\sigma>0$
suitable for all the $x$. Therefore we can
only say that we are dealing with pointwise convergence (not uniform convergence)
in our formula for the asymptotic expansion of $r_\sigma^*(x)$.
For practical applications, we do not need more than that.

The assumption that $p(x) \neq 0$ could also be viewed as problematic, but
if we go back to the definition of the DAE loss, then any point where $p(x) = 0$
would never produce a training example and would not even contribute in the
definition of the expectation of DAE loss. The correct value to assign to $r$ at such a point is
not well-defined.

\subsection{Relationship between Contractive Penalty and Denoising Criterion}

\setcounter{theorem}{0}
\begin{proposition}
\label{prp:DAE-connection-RCAE}
Let $p$ be the probability density function of the data.
Consider a $DAE$ using the expected quadratic loss
and corruption noise $N(x) = x+\epsilon$, with $\epsilon \sim \mathcal{N}\left(0, \sigma^2 I \right)$.
If we assume that the non-parametric solutions $r_\sigma(x)$ satistfies
\[
r_\sigma(x)=x+o(1) \hspace{1em} \textrm{as} \hspace{1em} \sigma \rightarrow 0,
\]
then we can rewrite the loss as 
\[
\mathcal{L}_{\textrm{DAE }} (r_\sigma) = \E\left[\|r_\sigma(x) - x\|^2_2 + \sigma^2 \left\|\frac{\partial r_\sigma(x)}{\partial x}\right\|^2_F\right]
+ o(\sigma^2) \hspace{1em} \textrm{as} \hspace{1em} \sigma \rightarrow 0
\]
where the expectation is taken with respect to $X$, whose distribution is given by $p$.
\end{proposition}

\begin{proof}
We can drop the $\sigma$ index from $r_\sigma$ to lighten the notation if we just keep in mind
that we are considering the particular family of solutions such that $r(x)-x = o(1)$.
With a Taylor expansion around $x$ we have that
\[
r(x+\epsilon) = r(x)+\frac{\partial r(x)}{\partial x}\epsilon + \mathcal{O}(\epsilon^2).
\]
The DAE loss involves taking the expectation with respect to $X$ and with respect to the noise $\epsilon$.
We substitute the Taylor expansion into ${\cal L}_{DAE}$, we express the norm as a dot product,
and we show how the expectation with respect to $\epsilon$ cancels out certain terms.
\begin{eqnarray}
&   & \mathbb{E}_{X}\mathbb{E}_{\epsilon}\left[\left\Vert x-\left(r(x)+\frac{\partial r(x)}{\partial x}\epsilon+\mathcal{O}(\epsilon^{2})\right)\right\Vert _{2}^{2}\right] \nonumber \\
& = & \mathbb{E}_{X}\left[\mathbb{E}_{\epsilon}\left[\|x-r(x)\|_{2}^{2}\right]-2\left(x-r(x)\right)^{T}\mathbb{E}_{\epsilon}\left[\frac{\partial r(x)}{\partial x}\epsilon+\mathcal{O}(\epsilon^{2})\right]+\mathbb{E}_{\epsilon}\left[\epsilon^{T}\frac{\partial r(x)}{\partial x}^{T}\frac{\partial r(x)}{\partial x}\epsilon + \mathcal{O}(\epsilon^{3})\right]\right] \nonumber \\
& = & \mathbb{E}_{X}\left[\left[\|r(x)-x\|_{2}^{2}\right]+o(1)\mathcal{O}(\sigma^{2})+\sigma^{2}\textrm{Tr}\left(\frac{\partial r(x)}{\partial x}^{T}\frac{\partial r(x)}{\partial x}\right)\right]+\mathcal{O}(\sigma^{3}) \nonumber \\
& = & \mathbb{E}_{X}\left[\|r(x)-x\|_{2}^{2}\right]+o(\sigma^{2})+\sigma^{2}\mathbb{E}_{X}\left[\left\Vert \frac{\partial r(x)}{\partial x}\right\Vert _{F}^{2}\right] \nonumber
\end{eqnarray}

The fact that we get a trace (rewritten as Frobenius norm) comes from the
assumption that the noise is Gaussian and is independant from $X$. This explains
why $\E\left[\epsilon \epsilon^T\right]=\sigma^2 I$ and $\E[\epsilon]=0$.
It also motivates many of our substitutions such as
$\mathbb{E}_{\epsilon}\left[\mathcal{O}(\epsilon^{k})\right]=\mathcal{O}(\sigma^{k})$.
\end{proof}

\subsection{Calculus of Variations}

\begin{theorem}
Let $p$ be a probability density function that is continuously differentiable
once and with support $\mathbb{R}^{d}$ (i.e. $\forall x\in\mathbb{R}^{d}$
we have $p(x)\neq0$). Let $\mathcal{L}_{\sigma}$ be the loss function defined
by
\[
\mathcal{L}_{\sigma}(r)=\int_{\mathbb{R}^{d}}p(x)\left[\left\Vert r(x)-x\right\Vert _{2}^{2}+{\sigma^2}\left\Vert \frac{\partial r(x)}{\partial x}\right\Vert _{F}^{2}\right]dx
\]
for $r:\mathbb{R}^{d}\rightarrow\mathbb{R}^{d}$ assumed to be
differentiable twice, and $0 \leq {\sigma}\in\mathbb{R}$ used as factor
to the penalty term.

Let $r_{\sigma}^*(x)$ denote the optimal function that minimizes
$\mathcal{L}_{\sigma}$. Then we have that
\[
r_{\sigma}^*(x) = x + {\sigma^2}\frac{\partial\log p(x)}{\partial x}+o({\sigma^2})\hspace{1em}\textrm{as}\hspace{1em}{\sigma} \rightarrow 0.
\]
Moreover, we also have the following expression for the derivative
\[
\frac{\partial r_{\sigma}^*(x)}{\partial x} = I + {\sigma^2}\frac{\partial^2\log p(x)}{\partial x^2} + o({\sigma^2})\hspace{1em}\textrm{as}\hspace{1em}{\sigma} \rightarrow 0.
\]
Both these asymptotic expansions are to be understood
in a context where we consider $\left\{r_{\sigma}^*(x)\right\}_{{\sigma} \geq 0}$
to be a family of optimal functions minimizing $\mathcal{L}_{\sigma}$
for their corresponding value of ${\sigma}$.
The asymptotic expansions are applicable point-wise in $x$, that is,
with any fixed $x$ we look at the behavior as ${\sigma} \rightarrow 0$.
\end{theorem}
\begin{proof}

This proof is done in two parts.
In the first part,
the objective is to get to equation (\ref{eq:rk-xk-equation-to-solve})
that has to be satisfied for the optimum solution.

We will leave out the ${\sigma}$ indices from the expressions involving $r(x)$
to make the notation lighter. We have a more important need for indices $k$ in $r_k(x)$
that denote the $d$ components of $r(x)\in\mathbb{R}^d$.

We treat ${\sigma}$ as given and constant for the first part of this proof.

In the second part we work out the asymptotic expansion
in terms of ${\sigma}$. We again work with the implicit
dependence of $r(x)$ on ${\sigma}$.

\textit{(part 1 of the proof)}

We make use of the Euler-Lagrange equation from
the Calculus of Variations. We would refer the reader to either
\citep{dacorogna2004introduction} or Wikipedia for more on the topic. Let
\[
f(x_{1},\ldots,x_{n},r,r_{x_{1}},\ldots,r_{x_{n}})=p(x)\left[\left\Vert r(x)-x\right\Vert _{2}^{2}+{\sigma^2}\left\Vert \frac{\partial r(x)}{\partial x}\right\Vert _{F}^{2}\right]
\]
where $x=(x_{1},\ldots,x_{d})$ , $r(x)=(r_{1}(x),\ldots,r_{d}(x))$
and $r_{x_{i}}=\frac{\partial f}{\partial x_{i}}$.


We can rewrite the loss $\mathcal{L}(r)$ more explicitly as

\begin{eqnarray}
\mathcal{L}(r) & = & \int_{\mathbb{R}^{d}}p(x)\left[\sum_{i=1}^{d}\left(r_{i}(x)-x_{i}\right)^{2}+{\sigma^2}\sum_{i=1}^{d}\sum_{j=1}^{d}\frac{\partial r_{i}(x)}{\partial x_{j}}^{2}\right]dx\nonumber \\
 & = & \sum_{i=1}^{d}\int_{\mathbb{R}^{d}}p(x)\left[\left(r_{i}(x)-x_{i}\right)^{2}+{\sigma^2}\sum_{j=1}^{d}\frac{\partial r_{i}(x)}{\partial x_{j}}^{2}\right]dx\label{eq:hinting-at-separation}
\end{eqnarray}

to observe that the components $r_{1}(x),\ldots,r_{d}(x)$ can each
be optimized separately.

\vphantom{}

The Euler-Lagrange equation to be satisfied at the optimal $r:\mathbb{R}^{d}\rightarrow\mathbb{R}^{d}$
is
\[
\frac{\partial f}{\partial r}=\sum_{i=1}^{d}\frac{\partial}{\partial x_{i}}\frac{\partial f}{\partial r_{x_{i}}}.
\]

In our situation, the expressions from that equation are given by
\[
\frac{\partial f}{\partial r}=2(r(x)-x)p(x)
\]

\[
\frac{\partial f}{\partial r_{x_{i}}}=2{\sigma^2} p(x)\left[\begin{array}{cccc}
\frac{\partial r_{1}}{\partial x_{i}} & \frac{\partial r_{2}}{\partial x_{i}} & \cdots & \frac{\partial r_{d}}{\partial x_{i}}\end{array}\right]^{T}
\]

\begin{eqnarray*}
\frac{\partial}{\partial x_{i}}\left(\frac{\partial f}{\partial r_{x_{i}}}\right) & = & 2{\sigma^2}\frac{\partial p(x)}{\partial x_{i}}\left[\begin{array}{cccc}
\frac{\partial r_{1}}{\partial x_{i}} & \frac{\partial r_{2}}{\partial x_{i}} & \cdots & \frac{\partial r_{d}}{\partial x_{i}}\end{array}\right]^{T}\\
 &  & +2{\sigma^2} p(x)\left[\begin{array}{cccc}
\frac{\partial^{2}r_{1}}{\partial x_{i}^{2}} & \frac{\partial^{2}r_{2}}{\partial x_{i}^{2}} & \cdots & \frac{\partial^{2}r_{d}}{\partial x_{i}^{2}}\end{array}\right]^{T}
\end{eqnarray*}

and the equality to be satisfied at the optimum becomes

\begin{equation}
(r(x)-x)p(x)={\sigma^2}\sum_{i=1}^{d}\left[\begin{array}{c}
\frac{\partial p(x)}{\partial x_{i}}\frac{\partial r_{1}}{\partial x_{i}}+p(x)\frac{\partial^{2}r_{1}}{\partial x_{i}^{2}}\\
\vdots\\
\frac{\partial p(x)}{\partial x_{i}}\frac{\partial r_{d}}{\partial x_{i}}+p(x)\frac{\partial^{2}r_{d}}{\partial x_{i}^{2}}
\end{array}\right].\label{eq:euler-lagrange-solution-Rd}
\end{equation}

As equation (\ref{eq:hinting-at-separation}) hinted, the expression
(\ref{eq:euler-lagrange-solution-Rd}) can be decomposed into the
different components $r_{k}(x):\mathbb{R}^{d}\rightarrow\mathbb{R}$
that make $r$. For $k=1,\ldots,d$ we get

\[
(r_{k}(x)-x_{k})p(x)={\sigma^2}\sum_{i=1}^{d}\left(\frac{\partial p(x)}{\partial x_{i}}\frac{\partial r_{k}(x)}{\partial x_{i}}+p(x)\frac{\partial^{2}r_{k}(x)}{\partial x_{i}^{2}}\right).
\]

As $p(x)\neq0$ by hypothesis, we can divide all the terms by $p(x)$
and note that $\frac{\partial p(x)}{\partial x_{i}}/p(x)=\frac{\partial\log p(x)}{\partial x_{i}}$.

\vphantom{}

We get

\begin{equation}\label{eq:rk-xk-equation-to-solve}
r_{k}(x)-x_{k}={\sigma^2}\sum_{i=1}^{d}\left(\frac{\partial\log p(x)}{\partial x_{i}}\frac{\partial r_{k}(x)}{\partial x_{i}}+\frac{\partial^{2}r_{k}(x)}{\partial x_{i}^{2}}\right).
\end{equation}

This first thing to observe is that when ${\sigma^2}=0$ the solution
is just $r_{k}(x)=x_{k}$, which translates into $r(x)=x$. This is
not a surprise because it represents the perfect reconstruction value
that we get when we the penalty term vanishes in the loss function.

\vphantom{}

\textit{(part 2 of the proof)}

This linear partial differential equation (\ref{eq:rk-xk-equation-to-solve}) can be used
as a recursive relation for $r_{k}(x)$
to obtain a Taylor series in ${\sigma^2}$.
The goal is to obtain an expression of the form
\begin{equation}\label{eq:rk-eq-xk-with-h-unknown}
r(x)
= x + {\sigma^2} h(x)
+ o({\sigma^2})\hspace{1em}\textrm{as}\hspace{1em}{\sigma^2}\rightarrow0
\end{equation}
where we can solve for $h(x)$
and for which we also have that
\[
\frac{\partial r(x)}{\partial x}
= I + {\sigma^2}\frac{\partial h(x)}{\partial x}
+ o({\sigma^2})\hspace{1em}\textrm{as}\hspace{1em}{\sigma^2}\rightarrow0.
\]
We can substitute in the right-hand side of equation (\ref{eq:rk-eq-xk-with-h-unknown})
the value for $r_{k}(x)$ that we get from equation (\ref{eq:rk-eq-xk-with-h-unknown})
itself. This substitution would be pointless in any other situation
where we are not trying to get a power series in terms of ${\sigma^2}$ around $0$.

\begin{eqnarray}
r_{k}(x) & = x_{k} & + {\sigma^2}\sum_{i=1}^{d}\left(\frac{\partial\log p(x)}{\partial x_{i}}\frac{\partial r_{k}(x)}{\partial x_{i}}+\frac{\partial^{2}r_{k}(x)}{\partial x_{i}^{2}}\right) \\
         & = x_{k} & + {\sigma^2}\sum_{i=1}^{d}\left(\frac{\partial\log p(x)}{\partial x_{i}}\frac{\partial}{\partial x_{i}}\left(x_{k}+{\sigma^2}\sum_{j=1}^{d}\left(\frac{\partial\log p(x)}{\partial x_{j}}\frac{\partial r_{k}(x)}{\partial x_{j}}+\frac{\partial^{2}r_{k}(x)}{\partial x_{j}^{2}}\right)\right)\right)\\
         &         & + {\sigma^2}\sum_{i=1}^{d}\frac{\partial^{2}r_{k}(x)}{\partial x_{i}^{2}}\\
         & = x_{k} & + {\sigma^2}\sum_{i=1}^{d}\frac{\partial\log p(x)}{\partial x_{i}}\mathbb{I}\left(i=k\right)+{\sigma^2}\sum_{i=1}^{d}\frac{\partial^{2}r_{k}(x)}{\partial x_{i}^{2}}\\
         &         & + {(\sigma^2})^{2}\sum_{i=1}^{d}\sum_{j=1}^{d}\left(\frac{\partial\log p(x)}{\partial x_{i}}\frac{\partial}{\partial x_{i}}\left(\frac{\partial\log p(x)}{\partial x_{j}}\frac{\partial r_{k}(x)}{\partial x_{j}}+\frac{\partial^{2}r_{k}(x)}{\partial x_{j}^{2}}\right)\right)\\
r_{k}(x) & = x_{k} & + {\sigma^2}\frac{\partial\log p(x)}{\partial x_{k}} + {\sigma^2}\sum_{i=1}^{d}\frac{\partial^{2}r_{k}(x)}{\partial x_{i}^{2}}+{\sigma^4}\rho({\sigma^2},x) \label{eq:rk_subs_in_subs_last_line}
\end{eqnarray}

Now we would like to get rid of that
${\sigma^2}\sum_{i=1}^{d}\frac{\partial^{2}r_{k}(x)}{\partial x_{i}^{2}}$
term by showing that it is a term that involves only powers of ${\sigma^4}$
or higher. We get this by showing what we get by differentiating the expression for $r_{k}(x)$
in line (\ref{eq:rk_subs_in_subs_last_line}) twice with respect to some $l$.

\[
\frac{\partial r_{k}(x)}{\partial x_{l}}=\mathbb{I}\left(i=l\right)+{\sigma^2}\frac{\partial^{2}\log p(x)}{\partial x_{l}\partial x_{k}}
+ {\sigma^2} \frac{\partial}{\partial x_{l}} \left( \sum_{i=1}^{d}\frac{\partial^{2}r_{k}(x)}{\partial x_{i}^{2}}+{\sigma^2} \rho({\sigma^2},x) \right)
\]

\[
\frac{\partial^{2}r_{k}(x)}{\partial x_{l}^{2}}={\sigma^2}\frac{\partial^{3}\log p(x)}{\partial x_{l}^{2}\partial x_{k}}
+ {\sigma^2} \frac{\partial}{\partial x_{l}^{2}} \left( \sum_{i=1}^{d}\frac{\partial^{2}r_{k}(x)}{\partial x_{i}^{2}}+{\sigma^2} \rho({\sigma^2},x) \right)
\]

Since ${\sigma^2}$ is a common factor in all the terms
of the expression of $\frac{\partial^{2}r_{k}(x)}{\partial x_{l}^{2}}$
we get what we needed. That is,
\[
r_{k}(x)=x_{k}+{\sigma^2}\frac{\partial\log p(x)}{\partial x_{k}}+{\sigma^4}\eta({\sigma^2},x).
\]
This shows that
\[
r(x)=x+{\sigma^2}\frac{\partial\log p(x)}{\partial x}+o({\sigma^2})\hspace{1em}\textrm{as}\hspace{1em}{\sigma^2}\rightarrow0
\]
and
\[
\frac{\partial r(x)}{\partial x}=I+{\sigma^2}\frac{\partial^{2}\log p(x)}{\partial x^{2}}+o({\sigma^2})\hspace{1em}\textrm{as}\hspace{1em}{\sigma^2}\rightarrow0
\]
which completes the proof.
\end{proof}

\subsection{Local Mean}
\label{sec:local-moments}

In preliminary work~\citep{Bengio-arxiv-moments-2012}, we studied how
the optimal reconstruction could possibly estimate so-called local moments.
We revisit this question here, with more appealing and precise results.

What previous work on denoising and contractive auto-encoders suggest is
that regularized auto-encoders can {\em capture the local structure of the
  density} through the value of the encoding (or reconstruction) function and
its derivative. In particular,
\citet{Rifai-icml2012,Bengio-arxiv-moments-2012} argue that the first and
second derivatives tell us in which directions it makes sense to randomly
move while preserving or increasing the density, which may be used to justify
sampling procedures. This motivates us here to study so-called local
moments as captured by the auto-encoder, and in particular the local mean,
following the definitions introduced in~\citet{Bengio-arxiv-moments-2012}.

\subsubsection{Definitions for Local Distributions}

Let $p$ be a continuous probability density function with support
$\mathbb{R}^{d}.$ That is, $\forall x\in\mathbb{R}^{d}$ we have
that $p(x)\neq0$. We define below the notion of a {\em local ball}
$B_{\delta}(x_{0})$, along with an associated {\em local density},
which is the normalized product of $p$ with the indicator for the ball:
\begin{eqnarray*}
B_{\delta}(x_{0}) & = & \left\{ x\hspace{1em}\textrm{s.t. }\left\Vert x-x_{0}\right\Vert _{2}<\delta\right\} \\
Z_{\delta}(x_{0}) & = & \int_{B_{\delta}(x_{0})}p(x)dx\\
p_{\delta}(x|x_{0}) & = & \frac{1}{Z_{\delta}(x_{0})}p(x)\mathbb{I}\left(x\in B_{\delta}(x_{0})\right)
\end{eqnarray*}

where $Z_{\delta}(x_{0})$ is the normalizing constant required to
make $p_{\delta}(x|x_{0})$ a valid pdf for a distribution centered
on $x_{0}$. The support of $p_{\delta}(x|x_{0})$ is
the ball of radius $\delta$ around $x_{0}$ denoted by $B_{\delta}(x_{0})$.
We stick to the $2$-norm
in terms of defining the balls $B_{\delta}(x_{0})$ used, but everything
could be rewritten in terms of another $p$-norm to have slightly
different formulas.


We use the following notation for what will be referred to as the 
first two \textit{local moments} (i.e. local mean and local covariance) of the
random variable described by $p_{\delta}(x|x_{0})$. 
\[
m_{\delta}(x_{0}) \stackrel{def}{=} \int_{\mathbb{R}^{d}}xp_{\delta}(x|x_{0})dx
\]
\[
C_{\delta}(x_{0}) \stackrel{def}{=} \int_{\mathbb{R}^{d}}(x-m_{\delta}(x_{0}))(x-m_{\delta}(x_{0}))^{T}p_{\delta}(x|x_{0})dx
\]

Based on these definitions, one can prove the following theorem.
\begin{theorem}
\label{thm:asymptotic-expr-m}
Let $p$ be of class $C^{3}$ and represent a probability density
function. Let $x_{0}\in\mathbb{R}^{d}$ with $p(x_{0})\neq0$. Then
we have that

\[
m_{\delta}(x_{0})=x_{0}+\delta^{2}\frac{1}{d+2}\left.\frac{\partial\log p(x)}{\partial x}\right|_{x_{0}}+o\left(\delta^{3}\right).
\]

\end{theorem}

This links the local mean of a density with the score associated with that density.
Combining this theorem with Theorem~\ref{thm:calcvarminloss}, we obtain that
the optimal reconstruction function $r^*(\cdot)$ also estimates the local mean:
\begin{equation}
m_{\delta}(x)-x = \frac{\delta^{2}}{{\sigma^2}(d+2)}\left(r^*(x) - x \right) \hspace{0.5em} + A(\delta) + \delta^2 B({\sigma^2})
\label{eq:m}
\end{equation}
for error terms $A(\delta), B({\sigma^2})$ such that
\begin{eqnarray*}
A(\delta) \in o(\delta^3) &  \textrm{as} \hspace{1em} \delta \rightarrow 0, \\
B({\sigma^2}) \in o(1) &  \textrm{as} \hspace{1em} {\sigma^2} \rightarrow 0.
\end{eqnarray*}

This means that we can loosely estimate the {\em direction} to the local mean by the direction of the reconstruction:
\begin{equation}
   m_{\delta}(x)-x \hspace{0.5em} \propto \hspace{0.5em} r^*(x) - x.
\label{eq:direction}
\end{equation}

\subsection{Asymptotic formulas for localised moments}

\begin{proposition}
\label{prop:asymptotic-expr-Z}

Let $p$ be of class $C^{2}$ and let $x_{0}\in\mathbb{R}^{d}$. Then
we have that

\[
Z_{\delta}(x_{0})=\delta^{d}\frac{\pi^{d/2}}{\Gamma\left(1+d/2\right)}\left[p(x_{0})+\delta^{2}\frac{\textrm{Tr}(H(x_{0}))}{2(d+2)}+o(\delta^{3})\right]
\]

where $H(x_{0})=\left.\frac{\partial^{2}p(x)}{\partial x^{2}}\right|_{x=x_{0}}$.
Moreover, we have that

\[
\frac{1}{Z_{\delta}(x_{0})}=\delta^{-d}\frac{\Gamma\left(1+d/2\right)}{\pi^{d/2}}\left[\frac{1}{p(x_{0})}-\delta^{2}\frac{1}{p(x_{0})^{2}}\frac{\textrm{Tr}(H(x_{0}))}{2(d+2)}+o(\delta^{3})\right].
\]

\end{proposition}

\begin{proof}

\begin{eqnarray*}
Z_{\delta}(x_{0}) & = & \int_{B_{\delta}(x_{0})}\left[p(x_{0})+\left.\frac{\partial p(x)}{\partial x}\right|_{x_{0}}(x-x_{0})+\frac{1}{2!}(x-x_{0})^{T}H(x_{0})(x-x_{0})\right.\\
 &  & \left.+\frac{1}{3!}D^{(3)}p(x_{0})(x-x_{0})+o(\delta^{3})\right]dx\\
 & = & p(x_{0})\int_{B_{\delta}(x_{0})}dx\,+\,0\,+\,\frac{1}{2}\int_{B_{\delta}(x_{0})}(x-x_{0})^{T}H(x_{0})(x-x_{0})dx\,+\,0\,+\, o(\delta^{d+3})\\
 & = & p(x_{0})\delta^{d}\frac{\pi^{d/2}}{\Gamma\left(1+d/2\right)}+\delta^{d+2}\frac{\pi^{d/2}}{4\Gamma\left(2+d/2\right)}\textrm{Tr}\left(H(x_{0})\right)+o(\delta^{d+3})\\
 & = & \delta^{d}\frac{\pi^{d/2}}{\Gamma\left(1+d/2\right)}\left[p(x_{0})+\delta^{2}\frac{\textrm{Tr}(H(x_{0}))}{2(d+2)}+o(\delta^{3})\right]
\end{eqnarray*}

We use Proposition \ref{proposition_ball_integration_yHy} to get
that trace come up from the integral involving $H(x_{0})$. The expression
for $1/Z_{\delta}(x_{0})$ comes from the fact that, for any $a,b>0$
we have that

\begin{eqnarray*}
\frac{1}{a+b\delta^{2}+o(\delta^{3})} & = & \frac{a^{-1}}{1+\frac{b}{a}\delta^{2}+o(\delta^{3})}=\frac{1}{a}\left(1-(\frac{b}{a}\delta^{2}+o(\delta^{3}))+o(\delta^{4})\right)\\
 & = & \frac{1}{a}-\frac{b}{a^{2}}\delta^{2}+o(\delta^{3})\hspace{1em}\textrm{as }\delta\rightarrow0.
\end{eqnarray*}

by using the classic result from geometric series where $\frac{1}{1+r}=1-r+r^{2}-\ldots$
for $|r|<1$. 

Now we just apply this to

\[
\frac{1}{Z_{\delta}(x_{0})}=\delta^{-d}\frac{\Gamma\left(1+d/2\right)}{\pi^{d/2}}\frac{1}{\left[p(x_{0})+\delta^{2}\frac{\textrm{Tr}(H(x_{0}))}{2(d+2)}+o(\delta^{3})\right]}
\]
 $ $and get the expected result.

\end{proof}

\begin{theorem}

Let $p$ be of class $C^{3}$ and represent a probability density
function. Let $x_{0}\in\mathbb{R}^{d}$ with $p(x_{0})\neq0$. Then
we have that

\[
m_{\delta}(x_{0})=x_{0}+\delta^{2}\frac{1}{d+2}\left.\frac{\partial\log p(x)}{\partial x}\right|_{x_{0}}+o\left(\delta^{3}\right).
\]

\end{theorem}

\begin{proof}

The leading term in the expression for $m_{\delta}(x_{0})$ is obtained
by transforming the $x$ in the integral into a $x-x_{0}$ to make
the integral easier to integrate.

\[
m_{\delta}(x_{0})=\frac{1}{Z_{\delta}(x_{0})}\int_{B_{\delta}(x_{0})}xp(x)dx=x_{0}+\frac{1}{z_{\delta}(x_{0})}\int_{B_{\delta}(x_{0})}(x-x_{0})p(x)dx.
\]

Now using the Taylor expansion around $x_{0}$

\begin{eqnarray*}
m_{\delta}(x_{0}) & = & x_{0}+\frac{1}{Z_{\delta}(x_{0})}\int_{B_{\delta}(x_{0})}(x-x_{0})\left[p(x_{0})+\left.\frac{\partial p(x)}{\partial x}\right|_{x_{0}}(x-x_{0})\right.\\
 &  & \left.+\frac{1}{2}(x-x_{0})^{T}\left.\frac{\partial^{2}p(x)}{\partial x^{2}}\right|_{x_{0}}(x-x_{0})+o(\left\Vert x-x_{0}\right\Vert ^{2})\right]dx.
\end{eqnarray*}

\vphantom{}

Remember that $\int_{B_{\delta}(x_{0})}f(x)dx=0$ whenever we have
a function $f$ is anti-symmetrical (or ``odd'') relative to the
point $x_{0}$ (i.e. $f(x-x_{0})=f(-x-x_{0})$). This applies to the
terms $(x-x_{0})p(x_{0})$ and $(x-x_{0})(x-x_{0})\left.\frac{\partial^{2}p(x)}{\partial x^{2}}\right|_{x=x_{0}}(x-x_{0})^{T}$. Hence we use Proposition \ref{proposition_ball_integration_vector}
to get

\begin{eqnarray*}
m_{\delta}(x_{0}) & = & x_{0}+\frac{1}{Z_{\delta}(x_{0})}\int_{B_{\delta}(x_{0})}\left[(x-x_{0})^{T}\left.\frac{\partial p(x)}{\partial x}\right|_{x_{0}}(x-x_{0})+o(\left\Vert x-x_{0}\right\Vert ^{3})\right]dx\\
 & = & x_{0}+\frac{1}{Z_{\delta}(x_{0})}\left(\delta^{d+2}\frac{\pi^{\frac{d}{2}}}{2\Gamma\left(2+\frac{d}{2}\right)}\right)\left.\frac{\partial p(x)}{\partial x}\right|_{x_{0}}+o(\delta^{3}).
\end{eqnarray*}

Now, looking at the coefficient in front of $\left.\frac{\partial p(x)}{\partial x}\right|_{x_{0}}$
in the first term, we can use Proposition \ref{prop:asymptotic-expr-Z} to rewrite it as

\[
\frac{1}{Z_{\delta}(x_{0})}\left(\delta^{d+2}\frac{\pi^{\frac{d}{2}}}{2\Gamma\left(2+\frac{d}{2}\right)}\right)=\delta^{-d}\frac{\Gamma\left(1+d/2\right)}{\pi^{d/2}}\left[\frac{1}{p(x_{0})}-\delta^{2}\frac{1}{p(x_{0})^{2}}\frac{\textrm{Tr}(H(x_{0}))}{2(d+2)}+o(\delta^{3})\right]\delta^{d+2}\frac{\pi^{\frac{d}{2}}}{2\Gamma\left(2+\frac{d}{2}\right)}
\]

\[
=\delta^{2}\frac{\Gamma\left(1+\frac{d}{2}\right)}{2\Gamma\left(2+\frac{d}{2}\right)}\left[\frac{1}{p(x_{0})}-\delta^{2}\frac{1}{p(x_{0})^{2}}\frac{\textrm{Tr}(H(x_{0}))}{2(d+2)}+o(\delta^{3})\right]=\delta^{2}\frac{1}{p(x_{0})}\frac{1}{d+2}+o(\delta^{3}).
\]

There is no reason the keep the $-\delta^{4}\frac{\Gamma\left(1+\frac{d}{2}\right)}{2\Gamma\left(2+\frac{d}{2}\right)}\frac{1}{p(x_{0})^{2}}\frac{\textrm{Tr}(H(x_{0}))}{2(d+2)}$
in the above expression because the asymptotic error from the remainder
term in the main expression is $o(\delta^{3})$. That would swallow
our exact expression for $\delta^{4}$ and make it useless.

We end up with

\[
m_{\delta}(x_{0})=x_{0}+\delta^{2}\frac{1}{d+2}\left.\frac{\partial\log p(x)}{\partial x}\right|_{x_{0}}+o(\delta^{3}).
\]

\end{proof}

\subsection{Integration on balls and spheres}

This result comes from \textit{Multi-dimensional Integration : Scary
Calculus Problems} from Tim Reluga (who got the results from \textit{How
to integrate a polynomial over a sphere} by Gerald B. Folland).

\begin{theorem}\label{thm:ball-integration-2-from-Reluga}

Let $B=\left\{ x\in\mathbb{R}^{d}\left|\sum_{j=1}^{d}x_{j}^{2}\leq1\right.\right\} $
be the ball of radius $1$ around the origin. Then

\[
\int_{B}\prod_{j=1}^{d}\left|x_{j}\right|^{a_{j}}dx=\frac{\prod\Gamma\left(\frac{a_{j}+1}{2}\right)}{\Gamma\left(1+\frac{d}{2}+\frac{1}{2}\sum a_{j}\right)}
\]

for any real numbers $a_{j}\geq0$.

\end{theorem}

\begin{corollary}\label{FollandCorollary1}

Let $B$ be the ball of radius $1$ around the origin. Then

\[
\int_{B}\prod_{j=1}^{d}x_{j}^{a_{j}}dx = \begin{cases}
\frac{ \prod \Gamma \left(\frac{a_{j}+1}{2}\right) }{ \Gamma \left( 1+\frac{d}{2}+\frac{1}{2}\sum a_{j}\right) } & \textrm{if all the \ensuremath{a_{j}} are even integers} \\
0 & \textrm{otherwise}
\end{cases}
\]

for any non-negative integers $a_{j}\geq0$. Note the absence of the
absolute values put on the $x_{j}^{a_{j}}$ terms.

\end{corollary}

\begin{corollary}\label{FollandCorollary2}

Let $B_{\delta}(0)\subset\mathbb{R}^{d}$ be the ball of radius $\delta$
around the origin. Then

\[
\int_{B_{\delta}(0)}\prod_{j=1}^{d}x_{j}^{a_{j}}dx=\begin{cases}
\delta^{d+\sum a_{j}}\frac{\prod\Gamma\left(\frac{a_{j}+1}{2}\right)}{\Gamma\left(1+\frac{d}{2}+\frac{1}{2}\sum a_{j}\right)} & \textrm{if all the \ensuremath{a_{j}}are even integers} \\
0 & \textrm{otherwise}
\end{cases}
\]

for any non-negative integers $a_{j}\geq0$. Note the absence of the
absolute values on the $x_{j}^{a_{j}}$ terms.

\end{corollary}

\begin{proof}

We take the theorem as given and concentrate here on justifying the
two corollaries.

Note how in Corollary \ref{FollandCorollary1} we dropped the absolute
values that were in the original Theorem \ref{thm:ball-integration-2-from-Reluga}.
In situations where at least one $a_{j}$ is odd, we have that the
function $f(x)=\prod_{j=1}^{d}x_{j}^{a_{j}}$ becomes odd in the sense
that $f(-x)=-f(x)$. Because of the symmetrical nature of the integration
on the unit ball, we get that the integral is 0 as a result of cancellations.

For Corollary \ref{FollandCorollary2}, we can rewrite the integral
by changing the domain with $y_{j}=x_{j}/\delta$ so that

\[
\delta^{-\sum a_{j}}\int_{B_{\delta}(0)}\prod_{j=1}^{d}x_{j}^{a_{j}}dx=\int_{B_{\delta}(0)}\prod_{j=1}^{d}\left(x_{j}/\delta\right)^{a_{j}}dx=\int_{B_{1}(0)}\prod_{j=1}^{d}y^{a_{j}}\delta^{d}dy.
\]

We pull out the $\delta^{d}$ that we got from the determinant of
the Jacobian when changing from $dx$ to $dy$ and Corollary \ref{FollandCorollary2}
follows.

\end{proof}

\begin{proposition}\label{proposition_ball_integration_vector}

Let $v\in\mathbb{R}^{d}$ and let $B_{\delta}(0)\subset\mathbb{R}^{d}$
be the ball of radius $\delta$ around the origin. Then

\[
\int_{B_{\delta}(0)}y<v,y>dy=\left(\delta^{d+2}\frac{\pi^{\frac{d}{2}}}{2\Gamma\left(2+\frac{d}{2}\right)}\right)\hspace{0.1cm}v
\]

where $<v,y>$ is the usual dot product.

\end{proposition}

\begin{proof}

We have that

\[
y<v,y>\hspace{1em}=\hspace{1em}\left[\begin{array}{c}
v_{1}y_{1}^{2}\\
\vdots\\
v_{d}y_{d}^{2}
\end{array}\right]
\]

which is decomposable into $d$ component-wise applications of Corollary
\ref{FollandCorollary2}. This yields the expected result with the
constant obtained from $\Gamma\left(\frac{3}{2}\right)=\frac{1}{2}\Gamma\left(\frac{1}{2}\right)=\frac{1}{2}\sqrt{\pi}.$

\end{proof}

\begin{proposition}\label{proposition_ball_integration_yHy}

Let $H\in\mathbb{R}^{d\times d}$ and let $B_{\delta}(x_{0})\subset\mathbb{R}^{d}$
be the ball of radius $\delta$ around $x_{0}\in\mathbb{R}^{d}$.
Then

\[
\int_{B_{\delta}(x_{0})}(x-x_{0})^{T}H(x-x_{0})dx=\delta^{d+2}\frac{\pi^{d/2}}{2\Gamma\left(2+d/2\right)}\,\textrm{trace}\left(H\right).
\]

\end{proposition}

\begin{proof}

First, by substituting $y=\left(x-x_{0}\right)/\delta$ we have that
this is equivalent to showing that

\[
\int_{B_{1}(0)}y^{T}Hydy=\frac{\pi^{d/2}}{2\Gamma\left(2+d/2\right)}\,\textrm{trace}\left(H\right).
\]

This integral yields a real number which can be written as

\[
\int_{B_{1}(0)}y^{T}Hy^{}dy=\int_{B_{1}(0)}\sum_{i,j}y_{i}H_{i,j}y_{j}dy=\sum_{i,j}\int_{B_{1}(0)}y_{i}y_{j}H_{i,j}dy.
\]

Now we know from Corollary \ref{FollandCorollary2} that this integral
is zero when $i\neq j$. This gives

\[
\sum_{i,j}H_{i,j}\int_{B_{1}(0)}y_{i}y_{j}dy=\sum_{i}H_{i,i}\int_{B_{1}(0)}y_{i}^{2}dy=\textrm{trace}\left(H\right)\frac{\pi^{d/2}}{2\Gamma\left(2+d/2\right)}.
\]

\end{proof}

\end{document}